\pdfoutput=1
\documentclass[nonacm, sigconf]{acmart}

\usepackage{amsmath, amsthm}
\usepackage{dsfont}
\usepackage{pgfplots}
\usepackage{xcolor}
\usepackage{xspace}
\usepackage{makecell, pifont} 

\renewcommand{\citeyearpar}{\citep} 

\allowdisplaybreaks
\pgfplotsset{compat=1.13}
\usetikzlibrary{shapes}

\newif\ifcomments\commentstrue

\newif\ifunabridged\unabridgedtrue

\newcommand{\Ycal}{\mathcal{Y}}
\newcommand{\Ycalcs}{\Ycal'}
\newcommand{\Ycalos}{\Ycal}
\newcommand{\Ycalps}{\hat{\Ycal}}

\newcommand{\Ycs}{Y'}
\newcommand{\ycs}{y'}
\newcommand{\Ytildecs}{\tilde{Y}'}
\newcommand{\ytildecs}{\tilde{y}'}
\newcommand{\Yos}{Y}
\newcommand{\yos}{y}
\newcommand{\Yps}{\hat{Y}}
\newcommand{\yps}{\hat{y}}

\newcommand{\binset}{\{0, 1\}}
\newcommand{\dtv}{d_{\mathrm{tv}}}
\newcommand{\dem}{d_{\mathrm{em}}}

\newcommand{\rhostar}{\rho^*_{\!\ell}}

\newcommand{\E}{\mathbb{E}}

\ifcomments
	\newcommand{\todo}[1]{{\color{red} TODO: #1}}
	\newcommand{\mct}[1]{{\color{red} MCT: #1}}
	\newcommand{\sam}[1]{{\color{blue} Sam: #1}}
\else
	\newcommand{\todo}[1]{}
	\newcommand{\mct}[1]{}
	\newcommand{\sam}[1]{}
\fi

\newtheorem{definition}{Definition}
\newtheorem{worldview}{Worldview}
\newtheorem{theorem}{Theorem}
\newtheorem{lemma}[theorem]{Lemma}

\title{Avoiding Disparity Amplification under Different Worldviews}

\author{Samuel Yeom}
\affiliation{\institution{Carnegie Mellon University}}
\email{syeom@cs.cmu.edu}

\author{Michael Carl Tschantz}
\affiliation{\institution{International Computer Science Institute}}
\email{mct@icsi.berkeley.edu}

\begin{abstract}
	We mathematically compare four competing definitions of group-level nondiscrimination: \emph{demographic parity}, \emph{equalized odds}, \emph{predictive parity}, and \emph{calibration}.
	Using the theoretical framework of Friedler et al., we study the properties of each definition under various \emph{worldviews}, which are assumptions about how, if at all, the observed data is biased.
	We argue that different worldviews call for different definitions of fairness, and we specify the worldviews that, when combined with the desire to avoid a criterion for discrimination that we call \emph{disparity amplification}, motivate demographic parity and equalized odds.
	We also argue that predictive parity and calibration are insufficient for avoiding disparity amplification because predictive parity allows an arbitrarily large inter-group disparity and calibration is not robust to post-processing.
	Finally, we define a worldview that is more realistic than the previously considered ones, and we introduce a new notion of fairness that corresponds to this worldview.
\end{abstract}

\begin{document}
	\maketitle
\section{Introduction} \label{sec:intro}

Researchers in the field of fair machine learning have proposed numerous tests for fairness, which focus on some quantitative aspect of a model that can be operationalized and checked using empirical, statistical, or program analytic methods.
These tests abstract away more subtle issues that are difficult to operationalize or too contentious to decide algorithmically, such as which groups or attributes should be protected and which cases should be treated as exceptions to general rules.
Our work sheds light on some of the possible assumptions behind and motivations for four common empirical tests that check for discrimination against groups.

The simplest of these tests, \emph{demographic parity}, checks whether the model gives the favorable outcome to two given groups of people at equal rates.
This test is an abstraction of the legal notion of \emph{disparate impact}, or \emph{indirect discrimination}, which in certain circumstances requires that some approximation of demographic parity hold.
Like disparate impact, demographic parity does not depend upon the intentions of the modeler, and it can flag a model that does not directly use the protected attribute if it instead uses another attribute that is correlated with the protected one.
However, demographic parity abstracts away disparate impact's exceptions for cases where there is sufficient justification for a disparity in outcomes, such as a \emph{business necessity}~\citep[e.g.,][]{grover1995business, barocas2016big}.
By completely abstracting away such exceptions, demographic parity may lead to models so inaccurate as to become useless, such as when predicting physical strength while requiring demographic parity on gender.

This impossibility of accuracy motivates moving away from demographic parity to tests that take the ground truth into account, allowing a degree of accuracy.
One such test, called \emph{equalized odds} by Hardt et al.~\citeyearpar{hardt2016equality}, requires equal false positive and false negative rates for each protected group.
Two other commonly used tests are \emph{predictive parity} and \emph{calibration}~\citep[e.g.,][]{chouldechova2017fair}, which impose conditions on the predictive values of the model for the protected groups.
Like demographic parity, all of these tests can be seen as abstractions of disparate impact in that they too examine disparities in outcomes, not how or why they were reached.
In contexts where accuracy can be considered a business necessity, these tests arguably provide a more refined abstraction of disparate impact than demographic parity does.

However, disagreement exists over which of these tests is the most appropriate, with some favoring calibration~\cite{compasrebuttal} and some favoring equalized odds~\cite{angwin2016propublica, hardt2016equality}.
It has been argued that adopting the calibration or equalized odds test corresponds to adopting the perspective of either the person using the classification or the person being classified, respectively~\cite{angwin2016propublica,narayanan2018translation}.
We provide a different lens on this disagreement and study the conditions under which each test allows the amplification of pre-existing disparities.

In some cases, the ``ground truth'' may be tainted by past discrimination, and consulting it will help perpetuate the discrimination.
In this work, we handle this issue by adopting the framework of Friedler et al.~\citeyearpar{friedler2016impossibility}, who make a distinction between the observed ground truth and the \emph{construct}, which is the attribute that is truly relevant for prediction.
For example, in the context of bail decisions, the construct could be whether a defendant commits a crime while out on bail, and the observed ground truth could be whether the defendant is rearrested for a crime.
Because the construct is usually unobservable, Friedler et al.\ introduce and analyze two assumptions, or \emph{worldviews}, about the construct:
Under the We're All Equal (WAE) worldview, there is no association between the construct and the protected attribute, and under the WYSIWYG worldview, the observations accurately reflect the construct.

By using the construct, we specify a natural criterion for discrimination.
This criterion, \emph{disparity amplification}, deals with the disparity in positive classification rates, which is a widely accepted measure of discriminatory effect in both law~\citep{eeoc1978} and computer science~\citep{calders2009building, calders2010three, kamishima2012fairness, zemel2013learning, feldman2015certifying, zafar2017fairness-aistats}.
It stipulates that a disparity in the output of the model is justified by a commensurate disparity in the construct, thereby allowing accurate models even when the base rates are different for different protected groups, as equalized odds, predictive parity, and calibration do. 
In addition, because it uses the construct, it does not depend upon the possibly biased ground truth.
Using the often unobservable construct can make testing for disparity amplification impossible; we argue that its value instead comes from its ability to organize the space of empirical tests.

In particular, one of our main contributions is our argument that the WAE and WYSIWYG worldviews, when combined with the desire to avoid disparity amplification, motivate demographic parity and equalized odds, respectively.
We thus shed light on why people may disagree about which empirical test of discrimination to apply in a particular setting:
Even if they agree on the need to avoid disparity amplification, they may disagree about the correct worldview to apply in that setting.
We also show that, regardless of the worldview and the base rates of the observed ground truth, predictive parity does not impose any restrictions on the extent to which a model amplifies disparity.
Calibration is more restrictive in this regard, but the common post-processing method of thresholding can amplify disparity to an arbitrary extent.
Since equalized odds is incompatible with predictive parity or calibration
~\citep{darlington1971another, chouldechova2017fair, kleinberg2017inherent}, this is an argument for the use of equalized odds instead of predictive parity or calibration.
Furthermore, we compare our approach to that of Zafar et al.~\citeyearpar{zafar2017fairness-www} in their work on \emph{disparate mistreatment}, or disparate misclassification rates, showing that the definition of disparity amplification can be modified to apply in their setting.

Although the WAE and WYSIWYG worldviews are useful for theoretical analysis, they are unlikely to be true in practice.
To remedy this issue, we introduce a family of hybrid worldviews that is parametrized by a measure of how biased the observed data is against a protected group of people.
This allows us to model many real-world situations by simply adjusting the parameter.
We then create a parametrized test for discrimination that corresponds to the new family of worldviews, showing how one can apply the analysis in our paper to more realistic scenarios.

Our most fundamental contribution is introducing a framework in which to motivate empirical tests in terms of construct-based criteria of discrimination and worldviews.
Disparity amplification is not the only relevant notion of discrimination, nor is it suitable in every context.
Indeed, there are many other aspects of discrimination that we do not address in this paper, such as intentional discrimination~\citep[\S II-A]{barocas2016big}, individual fairness~\citep{dwork2012fairness}, proxy discrimination~\citep{datta2017proxy}, delayed outcomes~\citep{liu2018delayed}, and affirmative action~\citep{kannan2019downstream}.
Future work may use our approach to tease out the assumptions implicit in these tests.

We view the discussed tests and disparity amplification as diagnostics that can lead to further investigations of potentially discriminatory behavior in a model.
As a result, we do not provide an algorithm for ensuring that a model does not have disparity amplification since, in our view, doing so would be treating the symptom rather than the cause.
Such algorithms can eliminate one aspect of discrimination, but may in the process create a model that is obviously discriminatory from another angle.
When a model does not satisfy a notion of nondiscrimination, it should be a starting point for investigation as to why.
While it could be that the learning algorithm is corrupt, it could also be due to a mismatch between the construct and the observed data, or a need for better features. 
No one test or criterion can ensure fairness~\citep{green2018myth}, and no single algorithm will be appropriate in all cases.

\section{Related Work}
Our work is most similar in structure to that of Heidari et al.~\citeyearpar{heidari2019moral}, who propose a unifying framework that reformulates some existing fairness definitions through the lens of equality of opportunity from political philosophy~\citep{rawls1971theory, roemer2002equality}.
They then propose a new fairness definition that is inspired by this lens.
Although we also present a unifying framework, our unification is through the lens of constructs and worldviews.

Friedler et al.~\citeyearpar{friedler2016impossibility} introduced the concept of the construct in fair machine learning.
Although they also use the construct in their definition of nondiscrimination, their definition uses the Gromov--Wasserstein distance and as a result is more difficult to compute and reason about.
One benefit of their approach is that it enables their treatment of fairness at both the individual level and the group level.
By contrast, we consider group nondiscrimination only, and this allows us to draw a parallel between the worldviews and the existing empirical tests of discrimination.

Barocas and Selbst~\citeyearpar{barocas2016big} discuss in detail the potential legal issues with discrimination in machine learning.
One widely consulted legal standard for detecting disparate impact is the \emph{four-fifths rule}~\citep{eeoc1978}.
The four-fifths rule is a guideline that checks whether the ratio of the rates of favorable outcomes for different demographic groups is at least four-fifths.
This guideline can be considered a relaxation of demographic parity, which would instead require that the ratio of the positive classification rates be exactly one.

The four-fifths rule has inspired the work of Feldman et al.~\citeyearpar{feldman2015certifying} and Zafar et al.~\citeyearpar{zafar2017fairness-aistats}, who deal with a generalization of the four-fifths rule, called the \emph{$p$\% rule}, in their efforts to remove disparate impact.
On the other hand, many others~\citep{calders2009building, calders2010three, kamishima2012fairness, zemel2013learning} consider the difference, rather than the ratio, of the positive classification rates.
Our discrimination criterion is a generalization of this difference-based measure, but it differs from the others in that it uses the construct rather than the observed data.

Other works in the field of fair machine learning deal with aspects of discrimination that are not well described by positive classification rates.
Hardt et al.~\citeyearpar{hardt2016equality} characterize nondiscrimination through \emph{equalized odds}, which requires that two measures of misclassification, false positive and false negative rates, be equal for all protected groups.
\emph{Calibration}, Chouldechova~\citeyearpar{chouldechova2017fair} points out, is widely accepted in the ``educational and psychological testing and assessment literature''.
In another work, Friedler et al.~\citeyearpar{friedler2019comparative} create a benchmark for empirically evaluating the consequences of imposing these and other definitions of fairness, finding that many, but not all, definitions lead to similar model behavior.

Dwork et al.~\citeyearpar{dwork2012fairness} formally define \emph{individual fairness} and give examples of cases where models are blatantly unfair at the individual level even though they satisfy demographic parity.
Although individual fairness is sometimes considered to be in conflict with group-based notions of fairness, Binns~\citeyearpar{binns2020apparent} argues otherwise, instead pointing to the difference in worldviews as the truly important factor.
He then lists demographic parity and calibration as corresponding to the WAE and WYSIWYG worldviews, respectively.
For the WYSIWYG worldview, he reasons that if calibration is satisfied, no applicant would receive a less favorable outcome than a less qualified applicant, assuming that the calibrated scores accurately describe the degree to which the applicant is qualified.
By contrast, in this paper we prove that equalized odds, but not calibration, is an effective way to avoid disparity amplification under the WYSIWYG worldview.

As mentioned previously, discriminatory effects can be justified if there is a sufficient reason.
For prediction tasks, it is natural to think of accuracy as a sufficient justification.
Zafar et al.~\citeyearpar{zafar2017fairness-aistats} handle this by solving an optimization problem to maximize fairness subject to some accuracy constraints.
This reflects the idea that a classifier is justified in sacrificing fairness for accuracy.
To a lesser extent, equalized odds, predictive parity, and calibration can also be thought of as motivated by the dual desires for accuracy and fairness.
Our approach to justification is also motivated by these desires, but we use the construct and say that a classifier is justified in predicting the construct correctly.

\section{Notation} \label{sec:notation}
In the framework introduced by Friedler et al.~\citeyearpar{friedler2016impossibility}, there are three spaces that describe the target attribute of a prediction model.
The \emph{construct space} represents the value of the attribute that is truly relevant for the prediction task.
This value is usually unobservable, so prediction models in a supervised learning problem are instead trained with a related measurable label, whose values reside in the \emph{observed space}.
Finally, the \emph{prediction space} (called \emph{decision space} by Friedler et al.) describes the output of the model.
We will use $\Ycs$, $\Yos$, and $\Yps$ as the random variables representing values from the construct, observed, and prediction spaces, respectively.
(See Figure~\ref{fig:spaces}.)

In addition, we will use $Z$ to denote the protected attribute at hand, and we will assume that $Z \in \binset$.
For example, if $Z$ is gender, the values 0 and 1 could represent male and female, respectively.
Although the input features $X = (X_1, \ldots, X_n)$ are also critical for both the training and the prediction of the model, they are rarely used in this paper.

\begin{description}
	\item[Example 1.] Some jurisdictions have started to use machine learning models to predict how much risk a criminal defendant poses~\citep{liptak2017sent}.
	Judges are then allowed to consider the risk score as one of many factors when making bail or sentencing decisions~\citep{loomis2016}.
	Using the three-space framework of Friedler et al.~\citeyearpar{friedler2016impossibility}, we can represent the risk score output by the model as $\Yps$.
	The model would be trained with the observation $\Yos$, which in this case may be recorded data about past criminal defendants and their failures to appear in court (bail) or recidivism (sentencing).
	These models would also be trained with features $X$ from the input space, such as age and criminal history.
	
	For sentencing decisions, presumably we want to know whether the defendant will commit another crime in the future, regardless of whether the defendant will be caught committing the crime.
	Therefore, we argue that the recorded recidivism rate $\Yos$ is merely a proxy for the actual reoffense rate $\Ycs$, which is the relevant attribute for the prediction task.
	There is evidence that Black Americans are arrested at a higher rate than White Americans for the same crime~\citep{mueller2018using}, so it is reasonable to suspect that $\Yos$ is a racially biased proxy for $\Ycs$.
	\item[Example 2.] Universities want the students that they admit to the university to be successful in the university ($\Ycs$).
	Because \emph{success} is a vague term that encompasses many factors, a model that predicts success in university would instead be trained with a more concrete measure, such as graduating within six years ($\Yos$).
	This model may take inputs such as a student's high-school grades and standardized test scores ($X$), and will output a prediction of how likely the student is to graduate within six years ($\Yps$).
	Admissions officers can then use this prediction to guide their decision about whether to admit the student.
\end{description}

It is important to note that the models in the above examples do not make the final decision and that human judgments are a major part of the decision process.
However, we are concerned about the fairness of the model rather than that of the entire decision process.
Thus, we focus on $\Yps$, the output of the model, rather than the final decision made using it.

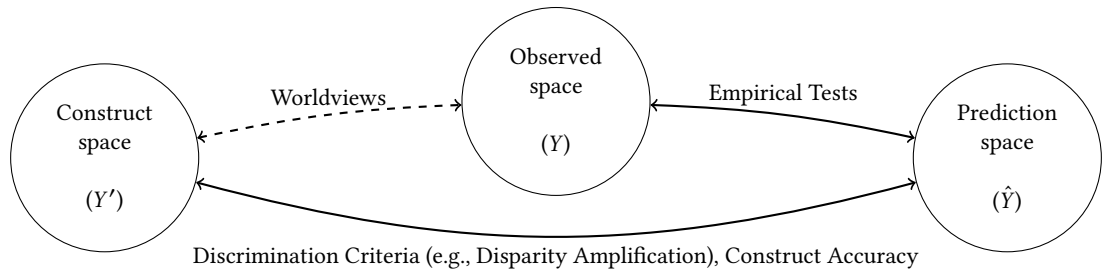
\begin{figure*}
	\begin{center}
		\begin{tikzpicture}
		\node[draw, ellipse, minimum height=2.5cm, minimum width=2.5cm, align=center] (cs) at (0,0) {Construct\\space\\[\baselineskip]($\Ycs$)};
		\node[draw, ellipse, minimum height=2.5cm, minimum width=2.5cm, align=center] (os) at (6,0.75) {Observed\\space\\[\baselineskip]($\Yos$)};
		\node[draw, ellipse, minimum height=2.5cm, minimum width=2.5cm, align=center] (ds) at (12,0) {Prediction\\space\\[\baselineskip]($\Yps$)};
		\draw[<->, dashed, thick] (cs) to[bend left=5] node[above] {Worldviews} (os);
		\draw[<->, thick] (os) to[bend left=5] node[above, align=center] {Empirical Tests} (ds);
		\draw[<->, thick] (cs) to[bend right=15] node[below] {Discrimination Criteria (e.g., Disparity Amplification), Construct Accuracy} (ds);
		\end{tikzpicture}
	\end{center}
	\caption{
		Three relevant spaces for prediction models.
		The space of input features $X = (X_1, \ldots, X_n)$ is not depicted here.
		The observed space and the prediction space are measurable, and the existing empirical tests (Definitions~\ref{def:demparity},~\ref{def:eqodds},~\ref{def:pp}) impose constraints on the relationship between the two spaces.
		On the other hand, the construct space is usually unobservable, so we must assume a particular worldview (e.g., Worldview~\ref{def:wae} or \ref{def:wysiwyg}) about how the construct space relates to the observed space, if at all.
		Then, we can define disparity amplification and construct accuracy, which relate the construct space to the prediction space.
	} \label{fig:spaces}
\end{figure*}

\section{Preliminary Definitions} \label{sec:prelim}
In this work, we use two notions of distance between two random variables that measure how different the random variables are.
When the random variables are categorical, we use the total variation distance.
\begin{definition}[Total Variation Distance] \label{def:dtv}
	Let $Y_0$ and $Y_1$ be categorical random variables with finite supports $\Ycal_0$ and $\Ycal_1$.
	Then, the \emph{total variation distance} between $Y_0$ and $Y_1$ is
	\[ \dtv(Y_0, Y_1) = \frac{1}{2} \sum_{y \in \Ycal_0 \cup \Ycal_1} \big|\Pr[Y_0{=}y] - \Pr[Y_1{=}y]\big|. \]
\end{definition}
In the special case where $Y_0, Y_1 \in \binset$, the total variation distance can also be expressed as $|\Pr[Y_0{=}1] - \Pr[Y_1{=}1]|$.

When the random variables are numerical, our notion of distance takes into account the magnitude of the difference in the numerical values.
The following definition assumes that the random variables are continuous, but a similar definition is applicable when they are discrete.
\begin{definition}[Earthmover Distance]
	Let $Y_0$ and $Y_1$ be continuous numerical random variables with probability density functions $p_0$ and $p_1$ defined over support $\Ycal$.
	Furthermore, let $\Gamma$ be the set of joint probability density functions $\gamma(u, v)$ such that $\int_{\Ycal} \gamma(u, v) \, dv = p_0(u)$ for all $u \in \Ycal$ and $\int_{\Ycal} \gamma(u, v) \, du = p_1(v)$ for all $v \in \Ycal$.
	Then, the \emph{earthmover distance} between $Y_0$ and $Y_1$ is
	\[ \dem(Y_0, Y_1) = \inf_{\gamma \in \Gamma} \int_{\Ycal} \int_{\Ycal} \gamma(u, v) \, d(u, v) \, du \, dv, \]
	where $d$ is a distance metric defined over $\Ycal$.
\end{definition}
The joint probability density function $\gamma$ has marginal distributions that correspond to $Y_0$ and $Y_1$.
Intuitively, if we use the graphs of the probability density functions $p_0$ and $p_1$ to represent mounds of sand, $\gamma$ corresponds to a transportation plan that dictates how much sand to transport in order to reshape the $p_0$ mound into the $p_1$ mound.
In particular, the value of $\gamma(u, v)$ is the amount of sand to be transported from $u$ to $v$.
The distance $d(u, v)$ can then be interpreted as the cost of transporting one unit of sand from $u$ to $v$, and the earthmover distance is simply the cost of the transportation plan $\gamma$ that incurs the least cost.

Now we define Lipschitz continuity.
\begin{definition} \label{def:lipschitz}
	Let $f: \Ycal \to \mathbb{R}$ be a function, and let $d$ be a distance metric defined over $\Ycal$.
	$f$ is \emph{$\rho$-Lipschitz continuous} if, for all $u, v \in \Ycal$,
	\begin{equation} \label{eqn:lipschitz}
	|f(u) - f(v)| \le \rho \cdot d(u, v).
	\end{equation}
\end{definition}

\subsection{Existing Empirical Tests of Discrimination}
Many fairness definitions for prediction models have been proposed previously, and here we restate four of them.
Because much of the prior work does not make the distinction between the construct space and the observed space, there is some ambiguity about whether $\Ycs$ or $\Yos$ is the appropriate variable to use these definitions.
Given that these works suggest that these definitions can be computed, we interpret them to be \emph{empirical tests} that can help verify whether a model is fair.
As a result, none of these definitions include the construct $\Ycs$.
In all four definitions, the probabilities are taken over random draws of data points from the data distribution, as well as any randomness used by the model.
\begin{definition}[Demographic Parity Test] \label{def:demparity}
	A model passes the \emph{demographic parity test} if, for all $\yps$,
	\[ \Pr[\Yps{=}\yps \mid Z{=}0] = \Pr[\Yps{=}\yps \mid Z{=}1]. \]
\end{definition}
\begin{definition}[Equalized Odds Test~\citep{hardt2016equality}] \label{def:eqodds}
	A model passes the \emph{equalized odds test} if, for all $\yos$ and $\yps$,
	\[ \Pr[\Yps{=}\yps \mid \Yos{=}\yos, Z{=}0] = \Pr[\Yps{=}\yps \mid \Yos{=}\yos, Z{=}1]. \]
\end{definition}
\begin{definition}[Predictive Parity Test~\citep{chouldechova2017fair}] \label{def:pp}
	A model passes the \emph{predictive parity test} if, for all $\yos$ and $\yps$,
	\[ \Pr[\Yos{=}\yos \mid \Yps{=}\yps, Z{=}0] = \Pr[\Yos{=}\yos \mid \Yps{=}\yps, Z{=}1]. \]
\end{definition}
Unlike the above three tests, the calibration test is only defined for binary observations, i.e., $\Yos \in \binset$.
\begin{definition}[Calibration Test~\citep{chouldechova2017fair}]
	\label{def:calib}
	A model with a binary $\Yos$ passes the \emph{calibration test} if, for all $\yps$ in the support of $\Yps$,
	\[ \Pr[\Yos{=}1 \mid \Yps{=}\yps, Z{=}0] = \Pr[\Yos{=}1 \mid \Yps{=}\yps, Z{=}1] = \yps. \]
\end{definition}

\subsection{Worldviews}
Our intuitive notion of discrimination involves the relationship between the construct space and the prediction space.
For example, consider the context of recidivism prediction described in Example~1.
Suppose that one group of people is much more likely to be arrested for the same crime than another group.
Then, the disparity in arrest rates can cause the recorded recidivism rate $\Yos$ to be biased, and a model trained using such $\Yos$ would likely learn to discriminate as a result.
If in fact the two groups have equal reoffense rates $\Ycs$, it would hardly be considered justified that one group tends to be given longer sentences as a result of the bias in $\Yos$.

However, because $\Ycs$ is typically unobservable, in practice we do not know whether $\Ycs$ is the same for both groups.
Therefore, to reason about discrimination using the construct space, we must make assumptions about the construct space.
Two such assumptions, or \emph{worldviews}, have previously been introduced by Friedler et al.~\citeyearpar{friedler2016impossibility} and are described below.
Our versions of these worldviews are simpler than the original because they are exact, whereas the original versions allow deviations by a parameter $\epsilon$.
\begin{worldview}[We're All Equal] \label{def:wae}
	Under the \emph{We're All Equal (WAE)} worldview, every group is identical with respect to the construct space. More formally, $\Ycs$ is independent of $Z$, i.e., $\Ycs \perp Z$.
\end{worldview}
\begin{worldview}[WYSIWYG] \label{def:wysiwyg}
	Under the \emph{What You See Is What You Get (WYSIWYG)} worldview, the observed space accurately reflects the construct space. More formally, $\Ycs = \Yos$.
\end{worldview}

\section{Construct Criteria} \label{sec:criteria}

We introduce two construct criteria for models.
By using the construct, these criteria must be combined with a worldview for application to a model.
Unlike the more readily applied empirical tests, construct criteria depend upon the attribute truly relevant to the classification task.

Here, we consider the case where $\Ycs$ and $\Yps$ are categorical (but not necessarily binary), and in Section~\ref{sec:general}
we generalize the definition to numerical $\Ycs$.

\subsection{Disparity Amplification} \label{sec:dispamp}
When $\Yps$ is binary, the size of a model's discriminatory effect is commonly measured by the difference in positive classification rates: $|\Pr[\Yps{=}1 \mid Z{=}0] - \Pr[\Yps{=}1 \mid Z{=}1]|$.
Output disparity generalizes this measure for the case of non-binary categorical $\Yps$.
\begin{definition}[Output Disparity] \label{def:outputdisp}
	Let the output $\Yps$ of a model be categorical.
	The \emph{output disparity} of the model is the quantity $\dtv(\Yps|Z{=}0, \Yps|Z{=}1)$.
\end{definition}
However, not all output disparities are bad in every context.
In particular, because we want the model to accurately reflect the construct, we allow an output disparity insofar as it can be explained by the inter-group disparity in $\Ycs$.
This happens when
\begin{equation} \label{eqn:no-dispamp}
\dtv(\Yps|Z{=}0, \Yps|Z{=}1) \leq \dtv(\Ycs|Z{=}0, \Ycs|Z{=}1).
\end{equation}
Since a model can have issues with discrimination that are not characterized by output disparity (see below), \eqref{eqn:no-dispamp} is not the conclusive definition of nondiscrimination.
Thus, we use the logical negation of \eqref{eqn:no-dispamp} as a criterion for one particular discrimination concern, which occurs when an output disparity is \emph{not} explained by $\Ycs$.
\begin{definition}[Disparity Amplification] \label{def:dispamp}
	Let $\Ycs$ and $\Yps$ be categorical.
	Then, a model exhibits \emph{disparity amplification} if
	\begin{equation} \label{eqn:dispamp}
	\dtv(\Yps|Z{=}0, \Yps|Z{=}1) > \dtv(\Ycs|Z{=}0, \Ycs|Z{=}1).
	\end{equation}
\end{definition}

\subsection{Construct Accuracy}
As mentioned in Section~\ref{sec:dispamp}, we want the output of the model to accurately reflect the value of $\Ycs$.
However, the simple accuracy measure $\Pr[\Ycs = \Yos]$ incentivizes the model to become more accurate on the larger protected group at the expense of becoming less accurate on the smaller protected group.
Therefore, we instead measure accuracy as the average of the accuracy on the two groups.
\begin{definition}[Construct Accuracy] \label{def:accuracy}
	The \emph{construct accuracy} of a model is
	\begin{equation} \label{eqn:accuracy}
	\textstyle \frac{1}{2} \big(\Pr[\Ycs{=}\Yps \mid Z{=}0] + \Pr[\Ycs{=}\Yps \mid Z{=}1]\big).
	\end{equation}
\end{definition}
\begin{definition}[Construct Optimality] \label{def:optimal}
	A model is \emph{construct optimal} if its construct accuracy is 1, i.e., its output $\Yps$ and the construct $\Ycs$ are always equal.
\end{definition}
Because the construct $\Ycs$ usually cannot be observed, construct accuracy usually cannot be measured or directly optimized for.
Even when it can measured, construct optimality would be rare since the quality of the features, data, or machine learning algorithm may preclude perfection.
As with disparity amplification, we introduce construct accuracy not to empirically measure it, but as a theoretical tool for analyzing discrimination.
In particular, note that equality holds in \eqref{eqn:no-dispamp} for every construct optimal model.
In other words, a construct optimal model displays the maximum amount of output disparity allowed by Definition~\ref{def:dispamp}.
On the other hand, if the output disparity is greater than the disparity in $\Ycs$, the model must be amplifying a disparity in a way that cannot be justified by the desire to achieve construct optimality.

The above definitions can be generalized to the setting where the range $\Ycalcs$ of the values that $\Ycs$ takes differs from the range $\Ycalps$ of $\Yps$.
If there exists a bijective mapping between $\Ycalcs$ and $\Ycalps$, we can use the mapping to characterize when a value from $\Ycalps$ accurately reflects a value from $\Ycalcs$.

\subsection{Limitations}

These criteria, separately or jointly, are neither necessary nor sufficient for fairness.
Technical criteria allow precision but elide the context-specific and social aspects of fairness~\cite{green2018myth}.

The criteria fail to be sufficient for fairness by not capturing forms of discrimination unrelated to output disparity.
For example, a model could have a higher misclassification rate for one group of people~\citep{zafar2017fairness-www}, which goes undetected by Definition~\ref{def:dispamp}.
(See Section~\ref{sec:misclassification} for discussion.)
Furthermore, by examining just a model's input/output behavior, the criteria cannot catch a model produced by an unacceptable process or performing unacceptable computations internally to reach its outputs.
For example, Datta et al.~\citeyearpar{datta2017proxy} show the impossibility of externally detecting whether a model internally reconstructs a sensitive attribute that it should not use.

We believe avoiding disparity amplification does better as a necessary condition for fairness, but limitations exist here as well.
For example, when correcting historical wrongs, it may be fair to amplify certain disparities that benefit an oppressed group.
Such cases also provide a counterexample to the necessity of construct accuracy.
In some cases, carefully selecting a historically informed construct can avoid violating our criteria while achieving a reparative goal.
However, some goals, such as achieving adequate representation for a group, cannot be expressed in terms of an individual-level construct.
Nevertheless, our criteria highlight when a model's behavior is suspicious enough to warrant an explanation and can serve as a basis for selecting between empirical tests.

\section{Using Criteria and Worldviews to Motivate Empirical Tests} \label{sec:motivate}

In this section, we use our construct criteria to analyze which worldviews motivate the existing empirical tests of discrimination.
If an empirical test does not guarantee the lack of disparity amplification, it may not be sufficient as an anti-discrimination measure as it effectively allows certain forms of discrimination.
On the other hand, if the test disallows a construct optimal model, the test may be too strict in a way that lowers the utility of the model.
Therefore, to argue that a worldview motivates an empirical test, we will prove the following two statements:
(a) Every model that passes the empirical test does not have disparity amplification, and
(b) every optimal model passes the empirical test.

We apply this reasoning to demographic parity (Definition~\ref{def:demparity}) and equalized odds (Definition~\ref{def:eqodds}), showing that the WAE and WYSIWYG worldviews, respectively, motivate these empirical tests.
More formally, we will prove statements (a) and (b) for every joint distribution of $\Ycs$, $\Yos$, $\Yps$, and $Z$ that is consistent with the worldview.
Table~\ref{tbl:summary} summarizes these results.

\begin{table}
	\caption{Summary of the results in Section~\ref{sec:motivate}.
	We say that a worldview motivates an empirical test if it precludes disparity amplification (Definition~\ref{def:dispamp}) but does not preclude a perfectly predictive model.
	The We're All Equal (WAE) worldview motivates the demographic parity test, and if the worldview does not hold, the demographic parity test tends to lower the utility of the model.
	The WYSIWYG worldview motivates the equalized odds test, and if the worldview does not hold, the equalized odds test allows models that have disparity amplification.
	Finally, regardless of the worldview, the predictive parity and calibration tests do not effectively prevent disparity amplification.
	Here, we assume that WAE and WYSIWYG do not hold simultaneously.
	}
	\label{tbl:summary}
	\begin{center}
	\begin{tabular}{@{}c|c|c@{}}
		& \makecell{We're All Equal \\ (Worldview~\ref{def:wae})}
		& \makecell{WYSIWYG \\ (Worldview~\ref{def:wysiwyg})} \\ \hline
		\makecell[t]{Demo.\ Parity \\ (Definition~\ref{def:demparity})}
		& \makecell[t]{\ding{52} \\ \footnotesize Theorem~\ref{thm:dispampwae}}
		& \makecell[t]{Necessarily suboptimal \\ \footnotesize Theorem~\ref{thm:notwae}} \\ \hline
		\makecell[t]{Equal.\ Odds \\ (Definition~\ref{def:eqodds})}
		& \makecell[t]{Amplification allowed \\ \footnotesize Theorem~\ref{thm:notwysiwyg}}
		& \makecell[t]{\ding{52} \\ \footnotesize Theorem~\ref{thm:dispampwysiwyg}} \\ \hline
		\makecell[t]{Predictive Parity \\ (Definition~\ref{def:pp})}
		& \multicolumn{2}{c}{\makecell[t]{Amplification allowed \\ \footnotesize Theorem~\ref{thm:pp}}} \\ \hline
		\makecell[t]{Calibration \\ (Definition~\ref{def:calib})}
		& \multicolumn{2}{c}{\makecell[t]{Not robust to post-processing \\ \footnotesize Theorem~\ref{thm:calib}}}
	\end{tabular}
	\end{center}
\end{table}

\subsection{Demographic Parity and WAE}
\begin{theorem} \label{thm:dispampwae}
	A model that passes the demographic parity test does not have disparity amplification under Definition~\ref{def:dispamp}. Moreover, if the WAE worldview holds, every construct optimal model satisfies demographic parity.
\end{theorem}
\begin{proof}
	By the definition of demographic parity, the left-hand side of \eqref{eqn:dispamp} is $\dtv(\Yps|Z{=}0, \Yps|Z{=}1) = 0$.
	Since the total variation distance is always nonnegative, demographic parity ensures the lack of disparity amplification.
	
	If the WAE worldview holds, we have $\Ycs \perp Z$, so every optimal model satisfies $\Yps \perp Z$.
	This implies demographic parity by Definition~\ref{def:demparity}.
\end{proof}

The first part of Theorem~\ref{thm:dispampwae} shows that we can guarantee that a model will not have disparity amplification by training it to pass the demographic parity test.
However, this does not mean that demographic parity is appropriate for every situation.
First, we remind the reader that the lack of disparity amplification does not mean that the model will be free of all issues related to discrimination.
In particular, disparity amplification is only designed to catch the type of discrimination akin to \emph{disparate impact}.
If the WAE worldview holds, demographic parity is the only way to avoid disparity amplification, so it makes sense to enforce demographic parity.
On the other hand, blindly enforcing demographic parity may introduce other forms of discrimination.
For example, the U.S.\ Supreme Court held in \textit{Ricci v.\ DeStefano}~\citeyearpar{ricci2009} that the prohibition against intentional discrimination can sometimes override the consideration of disparate impact, ruling that an employer unlawfully discriminated by discarding the results of a bona fide job-related test because of a racial performance gap.

Second, demographic parity can lower the utility of a model.
If the WAE worldview does not hold, $\dtv(\Ycs|Z{=}0, \Ycs|Z{=}1)$ is positive, and Theorem~\ref{thm:notwae} shows that any model that satisfies demographic parity must be suboptimal.
In fact, the more we deviate from the WAE worldview, the lower the maximum possible construct accuracy becomes.
\begin{theorem} \label{thm:notwae}
	If a model satisfies demographic parity, the construct accuracy of the model is at most $1 - \frac{1}{2} \dtv(\Ycs|Z{=}0, \Ycs|Z{=}1)$.
	Moreover, there exists a distribution of $\Yps$ that satisfies demographic parity and attains this construct accuracy.
\end{theorem}
To prove this theorem, we will use Lemma~\ref{lem:dtvmin}.
\begin{lemma} \label{lem:dtvmin}
	Let $Y_0$ and $Y_1$ be categorical random variables with finite supports $\Ycal_0$ and $\Ycal_1$.
	Then,
	\[ \sum_{y \in \Ycal_0 \cup \Ycal_1} \min\Big(\Pr[Y_0{=}y], \Pr[Y_1{=}y]\Big) = 1 - \dtv(Y_0, Y_1). \]
\end{lemma}
\begin{proof}[Proof of Lemma~\ref{lem:dtvmin}]
	For brevity, let $p_y = \Pr[Y_0{=}y]$ and $q_y = \Pr[Y_1{=}y]$.
	We can then rewrite the total variation distance in terms of $\max$ and $\min$.
	\begin{align*}
	2 \dtv(Y_0, Y_1) &= \textstyle \sum_{y \in \Ycal_0 \cup \Ycal_1} |p_y - q_y| \\
	&= \textstyle \sum_{y \in \Ycal_0 \cup \Ycal_1} \big(\max(p_y, q_y) - \min(p_y, q_y)\big).
	\end{align*}
	In addition, we have
	\[ \sum_{y \in \Ycal_0 \cup \Ycal_1} \Big(\max(p_y, q_y) + \min(p_y, q_y)\Big) = \sum_{y \in \Ycal_0 \cup \Ycal_1} (p_y + q_y) = 2. \]
	Subtracting the first equation from the second gives us $\sum \min(p_y, q_y) \allowbreak = 1 - \dtv(Y_0, Y_1)$, which is what we want.
\end{proof}
\begin{proof}[Proof of Theorem~\ref{thm:notwae}]
	We first prove the upper bound on the construct accuracy.
	Let $\Ycalcs$ and $\Ycalps$ be the supports of $\Ycs$ and $\Yps$, respectively.
	Then, by the law of total probability we have
	\[ \Pr[\Ycs{=}\ycs, \Yps{=}\ycs \mid Z{=}z] \le \min(\Pr[\Ycs{=}\ycs \mid Z{=}z], \Pr[\Yps{=}\ycs \mid Z{=}z]) \]
	for all $\ycs \in \Ycalcs \cup \Ycalps$ and $z \in \binset$.
	We then sum this over $\ycs$ and apply Lemma~\ref{lem:dtvmin} to get
	\begin{align*}
	&\Pr[\Ycs{=}\Yps \mid Z{=}z] \\
	&= \textstyle \sum_{\ycs \in \Ycalcs \cup \Ycalps} \Pr[\Ycs{=}\ycs, \Yps{=}\ycs \mid Z{=}z] \\
	&\le \textstyle \sum_{\ycs \in \Ycalcs \cup \Ycalps} \min\big(\Pr[\Ycs{=}\ycs \mid Z{=}z], \Pr[\Yps{=}\ycs \mid Z{=}z]\big) \\
	&= 1 - \dtv(\Ycs|Z{=}z, \Yps|Z{=}z) \\
	&= 1 - \dtv(\Ycs|Z{=}z, \Yps),
	\end{align*}
	where the last equality follows from our assumption that the model satisfies demographic parity.
	Therefore, the construct accuracy can be bounded as
	\begin{align*}
	&\textstyle \frac{1}{2} \big(\Pr[\Ycs{=}\Yps \mid Z{=}0] + \Pr[\Ycs{=}\Yps \mid Z{=}1]\big) \\
	&\textstyle \le \frac{1}{2} \big(1 - \dtv(\Ycs|Z{=}0, \Yps) + 1 - \dtv(\Ycs|Z{=}1, \Yps)\big) \\
	&\textstyle \le 1 - \frac{1}{2} \dtv(\Ycs|Z{=}0, \Ycs|Z{=}1),
	\end{align*}
	where the last inequality is an application of the triangle inequality.
	
	Now we construct a random variable $\Yps$ that satisfies demographic parity and attains this bound.
	When $Z{=}0$, we simply let $\Yps = \Ycs$, making the first term in \eqref{eqn:accuracy} equal to 1.
	When $Z{=}1$, we constrain the marginal distribution of $(\Yps|Z{=}1)$ to be the same as that of $(\Yps|Z{=}0) = (\Ycs|Z{=}0)$, and we make the joint distribution of $(\Ycs|Z{=}1)$ and $(\Yps|Z{=}1)$ a maximal coupling~\cite[pp.\ 19--20]{lindvall2002lectures}.
	Then, by the theorem in \cite[p.\ 19]{lindvall2002lectures}, such $\Yps$ attains the value of $1 - \dtv(\Yps|Z{=}1, \Ycs{=}1) = 1 - \dtv(\Ycs|Z{=}0, \Ycs|Z{=}1)$ for the second term of \eqref{eqn:accuracy}.
	This means that the construct accuracy, which is the average of the two terms, is $1 - \frac{1}{2} \dtv(\Ycs|Z{=}0, \Ycs|Z{=}1)$, which is what we want.
	Moreover, $(\Yps|Z{=}1)$ and $(\Yps|Z{=}0)$ have the same distribution, so $\Yps$ satisfies demographic parity.
\end{proof}
Theorems~\ref{thm:dispampwae} and \ref{thm:notwae} demonstrate that the WAE worldview, combined with the desire to avoid disparity amplification while retaining the utility of models, motivates the demographic parity test.

\subsection{Equalized Odds and WYSIWYG}
We now argue that a similar relationship exists between the equalized odds test and the WYSIWYG worldview.
\begin{theorem} \label{thm:dispampwysiwyg}
	If the WYSIWYG worldview holds, a model that passes the equalized odds test does not have disparity amplification under Definition~\ref{def:dispamp}. Moreover, if the WYSIWYG worldview holds, every construct optimal model satisfies equalized odds.
\end{theorem}
\begin{proof}
	Let $\Ycalcs$ and $\Ycalps$ be the supports of $\Ycs$ and $\Yps$, respectively.
	Applying the WYSIWYG worldview to the definition of equalized odds, we get $\Pr[\Yps{=}\yps \mid \Ycs{=}\ycs, Z{=}0] = \Pr[\Yps{=}\yps \mid \Ycs{=}\ycs, Z{=}1] = \Pr[\Yps{=}\yps \mid \Ycs{=}\ycs]$ for all $\ycs \in \Ycalcs$ and $\yps \in \Ycalps$.
	Therefore, we have
	\begin{multline*}
	\textstyle \dtv(\Yps|Z{=}0, \Yps|Z{=}1) \\
	\shoveleft{\textstyle = \frac{1}{2} \sum_{\yps \in \Ycalps} \big|\Pr[\Yps{=}\yps \mid Z{=}0] - \Pr[\Yps{=}\yps \mid Z{=}1]\big|} \\
	\shoveleft{\textstyle = \frac{1}{2} \sum_{\yps \in \Ycalps} \Big|\sum_{\ycs \in \Ycalcs} \Pr[\Yps{=}\yps \mid \Ycs{=}\ycs]} \\
	\shoveright{\textstyle \cdot \big(\Pr[\Ycs{=}\ycs \mid Z{=}0] - \Pr[\Ycs{=}\ycs \mid Z{=}1]\big)\Big|} \\
	\shoveleft{\textstyle  \le \frac{1}{2} \sum_{\yps \in \Ycalps} \sum_{\ycs \in \Ycalcs} \Pr[\Yps{=}\yps \mid \Ycs{=}\ycs]} \\
	\shoveright{\textstyle \cdot \big|\Pr[\Ycs{=}\ycs \mid Z{=}0] - \Pr[\Ycs{=}\ycs \mid Z{=}1]\big|} \\
	\shoveleft{\textstyle = \frac{1}{2} \sum_{\ycs \in \Ycalcs} \Big(\big|\Pr[\Ycs{=}\ycs \mid Z{=}0] - \Pr[\Ycs{=}\ycs \mid Z{=}1]\big|} \\
	\shoveright{\textstyle \cdot \sum_{\yps \in \Ycalps} \Pr[\Yps{=}\yps \mid \Ycs{=}\ycs]\Big)} \\
	\shoveleft{\textstyle = \frac{1}{2} \sum_{\ycs \in \Ycalcs} \big|\Pr[\Ycs{=}\ycs \mid Z{=}0] - \Pr[\Ycs{=}\ycs \mid Z{=}1]\big|} \\
	\shoveleft{\textstyle = \dtv(\Ycs|Z{=}0, \Ycs|Z{=}1). \hfill}
	\end{multline*}
	This concludes the proof of the first statement.
	
	For an optimal model, we have $\Yps = \Ycs = \Yos$ by the WYSIWYG worldview.
	Because $\Yos$ fully determines the value of $\Yps$, Definition~\ref{def:eqodds} implies that every optimal model satisfies equalized odds.
\end{proof}

On the other hand, our intuition is that when the observation process is biased, and WYSIWYG does not hold, treating the observation $\Yos$ as accurate, as implicit with equalized odds, may lead to a failure to pass our construct-based criterion.
We prove as much:
\begin{theorem} \label{thm:notwysiwyg}
	If the WYSIWYG worldview does not hold, a model passing the equalized odd test can still have disparity amplification.
\end{theorem}
\begin{proof}
	We show that there exists a joint distribution of $\Ycs$, $\Yos$, $\Yps$, and $Z$ such that a model with equalized odds still has disparity amplification.
	Many models with equalized odds have nonzero output disparity, i.e., $\dtv(\Yps|Z{=}0, \Yps|Z{=}1) > 0$.
	Consider any such model.
	Since the WYSIWYG worldview does not hold, we have no guarantee that $\Ycs$ will resemble $\Yos$ in any way.
	Therefore, the equalized odds requirement does not restrict the distribution of $\Ycs$, and the model can have disparity amplification if $\dtv(\Ycs|Z{=}0, \Ycs|Z{=}1)$ is small enough.
\end{proof}

\subsection{Predictive Parity}

Under the WYSIWYG worldview, optimal models pass the predictive parity test, but any model that passes the test must satisfy $\dtv(\Yps|Z{=}0, \Yps|Z{=}1) \ge \dtv(\Ycs|Z{=}0, \Ycs|Z{=}1)$, as can be seen from switching $\Ycs$ and $\Yps$ in the proof of the first part of Theorem~\ref{thm:dispampwysiwyg}.
The inequality here is in the opposite direction of that in \eqref{eqn:no-dispamp}, so the predictive parity test does not place any upper bound on the output disparity of $\Yps$ and guarantees that it is equal to that of $\Ycs$ or amplified beyond this limit.
In fact, the following theorem shows that, regardless of the worldview and the base rates of $\Yos$, even a model with almost the maximum output disparity can still pass the predictive parity test.
\begin{theorem} \label{thm:pp}
	Let $\Yos$ be a categorical random variable with finite support such that $\Pr[\Yos{=}\yos \mid Z{=}z]$ is positive for all $\yos$ and $z$.
	Then, for any sufficiently small $\epsilon > 0$, there exists a model that passes the predictive parity test such that $\dtv(\Yps|Z{=}0, \Yps|Z{=}1) = 1 - \epsilon$.
\end{theorem}
\begin{proof}
	The main idea behind the proof is that the model simply outputs the value of $Z$.
	However, because predictive parity is not well-defined if $\Pr[\Yps{=}\yps, Z{=}z] = 0$ for any $\yps$ and $z$, we must allow the model to output the other value with some very small probability.
	More specifically, we construct a model such that
	\[ \Pr[\Yps{=}\yps \mid Z{=}z] = \begin{cases}
	1 - \frac{\epsilon}{2}, & \text{if } \yps = z \\
	\frac{\epsilon}{2}, & \text{if } \yps \neq z.
	\end{cases} \]
	We can choose which values our constructed model outputs, so assume without loss of generality that $\Yps \in \binset$.
	
	Let $\Ycalos$ be the support of $\Yos$.
	By the predictive parity test, we have $\Pr[\Yos{=}\yos \mid \Yps{=}\yps, Z{=}0] = \Pr[\Yos{=}\yos \mid \Yps{=}\yps, Z{=}1] = \Pr[\Yos{=}\yos \mid \Yps{=}\yps]$ for all $\yos \in \Ycalos$ and $\yps \in \binset$.
	Let $p_{\yos\yps} = \Pr[\Yos{=}\yos \mid \Yps{=}\yps]$.
	Our goal is to find the values of $p_{\yos0}$ and $p_{\yos1}$ that are consistent with the fixed observed probabilities $\Pr[\Yos{=}\yos \mid Z{=}0]$ and $\Pr[\Yos{=}1 \mid Z{=}1]$.
	
	By the law of total probability, our model must satisfy
	\[\begin{pmatrix} \Pr[\Yos{=}\yos \mid Z{=}0] \\ \Pr[\Yos{=}\yos \mid Z{=}1] \end{pmatrix}
	= \begin{pmatrix} 1 - \frac{\epsilon}{2} & \frac{\epsilon}{2} \\ \frac{\epsilon}{2} & 1 - \frac{\epsilon}{2} \end{pmatrix}
	\begin{pmatrix} p_{\yos0} \\ p_{\yos1} \end{pmatrix}. \]
	Solving for $p_{\yos0}$ and $p_{\yos1}$, we see that they converge to $\Pr[\Yos{=}\yos \mid Z{=}0]$ and $\Pr[\Yos{=}\yos \mid Z{=}1]$, respectively, as $\epsilon$ approaches zero.
	By assumption, these probabilities are positive.
	Since $\Ycalos$ is finite, this means that there exists a small enough $\epsilon > 0$ such that $p_{\yos0}, p_{\yos1} > 0$ for all $\yos \in \Ycalos$.
	Moreover, it is easy to verify that $\sum_{\yos \in \Ycalos} p_{\yos0} = \sum_{\yos \in \Ycalos} p_{\yos1} = 1$, making them valid probability distributions.
	
	Now, when given $\Yos{=}\yos$ and $Z{=}z$, our model can output $\Yps{=}\yps$ with probability
	\[ \Pr[\Yps{=}\yps \mid \Yos{=}\yos, Z{=}z] = \frac{p_{\yos\yps} \cdot \Pr[\Yps{=}\yps \mid Z{=}z]}{\Pr[\Yos{=}\yos \mid Z{=}z]}, \]
	where $\Pr[\Yps{=}\yps \mid Z{=}z]$ is either $\frac{\epsilon}{2}$ or $1 - \frac{\epsilon}{2}$ depending on whether $\yps = z$.
\end{proof}
Because the predictive parity test allows models, such as the one we constructed in the above proof, that clearly amplify disparity, it is unsuitable for ensuring nondiscrimination as characterized by output disparity.

\subsection{Calibration}
Compared to the predictive parity test, the calibration test imposes an additional requirement that the output of the model must be the correct probability.
Theorem~\ref{thm:calib} shows this additional requirement limits the model behavior by the disparity in observed values, ruling out the model described in the proof of Theorem~\ref{thm:pp}.
\begin{theorem} \label{thm:calib}
	If the WYSIWYG worldview holds, a model that passes the calibration test satisfies $\dtv(\Ycs|Z{=}0, \Ycs|Z{=}1) = |\E[\Yps \mid Z{=}0] - \E[\Yps \mid Z{=}1]|$.
	Moreover, if the WYSIWYG worldview holds, every construct optimal model with binary $\Yos$ satisfies calibration.
\end{theorem}
\begin{proof}
	Combining the definition of calibration with the WYSIWYG worldview, we get a binary $\Ycs$ with $\Pr[\Ycs{=}1 \mid \Yps{=}\yps, Z{=}0] = \yps$.
	Therefore, we have
	\begin{align*}
	\Pr[\Ycs{=}1 \mid Z{=}0]
	&\textstyle = \sum_{\yps \in \Ycalps} \Pr[\Ycs{=}1 \mid \Yps{=}\yps, Z{=}0] \cdot \Pr[\Yps{=}\yps \mid Z{=}0] \\
	&\textstyle = \sum_{\yps \in \Ycalps} \yps \cdot \Pr[\Yps{=}\yps \mid Z{=}0] \\
	&= \E[\Yps \mid Z{=}0],
	\end{align*}
	and a similar statement holds for $Z{=}1$.
	
	Since $\Ycs$ is binary, the construct disparity then becomes
	\begin{align*}
	\dtv(\Ycs|Z{=}0, \Ycs|Z{=}1)
	&= |\Pr[\Ycs{=}1 \mid Z{=}0] - \Pr[\Ycs{=}1 \mid Z{=}1]| \\
	&= |\E[\Yps \mid Z{=}0] - \E[\Yps \mid Z{=}1]|,
	\end{align*}
	which is what we want for the first statement.
	
	To prove the second statement, note that an optimal model satisfies $\Yps = \Ycs = \Yos$ by the WYSIWYG worldview.
	Then, for binary $\Yos \in \binset$ it is easy to verify that calibration holds.
\end{proof}
Unlike Theorems~\ref{thm:dispampwae} and \ref{thm:dispampwysiwyg}, which bound the \emph{total variation distance} between the outputs by the disparity in the construct, this theorem bounds only the difference in the \emph{expected values} of the outputs.
This contrast is significant because expected value, unlike total variation distances, are not robust to post-processing.
%
%
We demonstrate this issue with an example where $\Ycs = \Yos$ and $Z$ are independent and uniformly random binary variables and the model sets the value of $\Yps$ as follows:
if $Z = 0$, then $\Yps = 0.5$;
if $Z = 1$ and $\Yos = 1$, then $\Yps = 0.5+\epsilon$ for some small positive constant $\epsilon$; and
if $Z = 1$ and $\Yos = 0$, then $\Yps = 0$ with probability $\frac{2\epsilon}{0.5+\epsilon}$ and $\Yps = 0.5+\epsilon$ otherwise.
Some computation reveals that this model passes the calibration test, with all of the $Z=0$ group receiving a prediction of $0.5$ and the vast majority of the $Z=1$ group receiving $0.5+\epsilon$.
However, in practice the predictions are often post-processed with a threshold because it is impossible to, say, admit half of a student.
Therefore, although the inter-group difference in the model predictions is small, it can be amplified if the threshold is set between $0.5$ and $0.5 + \epsilon$.
In this case, the resulting decision is almost perfectly correlated with $Z$ and exhibits disparity amplification.

As a result, in the rest of the paper we focus on equalized odds rather than predictive parity or calibration.
We leave as future work the identification of a discrimination criterion and a worldview that together motivate the predictive parity or calibration test.

\section{Connection to Misclassification}
\label{sec:misclassification}

Here, we show that the definition of disparity amplification is closely related to that given by Zafar et al.~\citeyearpar{zafar2017fairness-www} in their treatment of disparate misclassification rates.
First, we motivate the issue of disparate misclassification rates with an example.
Let $\Ycs$ and $Z$ be independent and uniformly random binary variables.
If $\Yps = \Ycs \oplus Z$, where $\oplus$ is the XOR, both protected groups are given the positive label exactly half of the time, so there is no output disparity.
However, one group always receives the correct classification and the other always receives the incorrect classification, so the disparity in the misclassification rates is as large as it can be.
This shows that a lack of disparity amplification does not imply a lack of disparity in misclassification rates.

Conversely, a lack of disparity in misclassification rates does not imply a lack of disparity amplification.
To see this, modify the above example so that $\Yps = Z$ instead.
Now, both groups have half of its members misclassified since $Z$ is independent of $\Ycs$, so they have the same overall misclassification rate.
On the other hand, we have $\dtv(\Ycs|Z{=}0, \Ycs|Z{=}1) = \dtv(\Ycs, \Ycs) = 0$ and $\dtv(\Yps|Z{=}0, \Yps|Z{=}1) = \dtv(Z|Z{=}0, Z|Z{=}1) = 1$.
Thus, $\Yps$ has disparity amplification.

However, we can still find a connection between misclassification parity and disparity amplification.
Let $C$ be the indicator $\mathds{1}(\Ycs = \Yps)$, and replace $\Yps$ with $C$ in the definition of output disparity (Definition~\ref{def:outputdisp}).
Since $C$ is binary, the resulting expression $\dtv(C|Z{=}0, C|Z{=}1)$ is simply the difference in the misclassification rates.
We would like to compare this value to some measure of disparity in the construct space.
Since our standard measure of $\dtv(\Ycs|Z{=}0, \Ycs|Z{=}1)$ does not necessarily justify inter-group differences in $C$, it may not be a correct measure to use.
Exploring what measures provide justification for disparate misclassification rates is interesting future work.

\section{Hybrid Worldviews} \label{sec:hybrid}
So far, we have assumed either the WAE or the WYSIWYG worldview.
While these worldviews are interesting from a theoretical perspective, in practice it is unlikely that these worldviews hold.

In this section, we propose a family of more realistic worldviews for the case where $\Ycs$ and $\Yos$ are categorical.
As we have depicted in Figure~\ref{fig:spaces}, worldviews describe the relationship between the construct and observed spaces.
Because our definition of disparity amplification has to do with inter-group disparities, here we focus specifically on the inter-group disparities in $\Ycs$ and $\Yos$.
Note that the WAE worldview has the effect of assuming that none of the disparity in $\Yos$ is explained by $\Ycs$.
By contrast, under the WYSIWYG worldview, all of the disparity in $\Yos$ is explained by $\Ycs$.
Described below is the $\alpha$-Hybrid worldview, which is a family of worldviews that occupy the space between the two extremes of WAE and WYSIWYG.
\begin{worldview}[$\alpha$-Hybrid] \label{def:hybrid}
	Let $\alpha \in [0, 1]$.
	Under the \emph{$\alpha$-Hybrid} worldview, exactly an $\alpha$ fraction of the disparity in $\Yos$ is explained by $\Ycs$.
	More formally,
	\begin{equation} \label{eqn:hybrid}
	\dtv(\Ycs|Z{=}0, \Ycs|Z{=}1) = \alpha \cdot \dtv(\Yos|Z{=}0, \Yos|Z{=}1)
	\end{equation}
\end{worldview}
While the WAE worldview is equivalent to the 0-Hybrid worldview,
the relationship between the WYSIWYG and 1-Hybrid worldviews is only unidirectional.
Although the WYSIWYG worldview implies the 1-Hybrid worldview, there are plenty of ways to satisfy $\dtv(\Ycs|Z{=}0, \Ycs|Z{=}1) = \dtv(\Yos|Z{=}0, \Yos|Z{=}1)$ even when the equality $\Ycs = \Yos$ does not hold.
If we wanted to make the relationship bidirectional, we could instead have assumed that $\Ycs$ can be broken down into two components, one of which satisfies WAE and the other WYSIWYG.
However, this would mean that every component of $\Ycs$ is either equal with respect to $Z$ (WAE) or observable (WYSIWYG), whereas in practice many inter-group disparities in the construct space are not easily observable.
Thus, to make the $\alpha$-Hybrid worldview more realistic, we sacrifice one direction of the relationship between the WYSIWYG and 1-Hybrid worldviews.

Now we introduce the $\alpha$-disparity test and prove that it corresponds to the $\alpha$-Hybrid worldview.
Unlike the demographic parity and equalized odds tests, the $\alpha$-disparity test is parametrized and therefore can be applied to various real-world situations.
\begin{definition}[$\alpha$-Disparity Test] \label{def:disparity}
	A model passes the \emph{$\alpha$-disparity test} if
	\begin{equation} \label{eqn:disparity}
	\dtv(\Yps|Z{=}0, \Yps|Z{=}1) \le \alpha \cdot \dtv(\Yos|Z{=}0, \Yos|Z{=}1).
	\end{equation}
\end{definition}
\begin{theorem} \label{thm:dispamphybrid}
	If the $\alpha$-Hybrid worldview holds, a model that passes the $\alpha$-disparity test does not have disparity amplification under Definition~\ref{def:dispamp}.
	Moreover, if the $\alpha$-Hybrid worldview holds, every construct optimal model satisfies the $\alpha$-disparity test.
\end{theorem}
\begin{proof}
	To prove the first part of the theorem, we simply combine the inequality guaranteed by the $\alpha$-disparity test under \eqref{eqn:disparity} with the equation that defines the $\alpha$-Hybrid worldview under \eqref{eqn:hybrid}.
	We get
	\[ \dtv(\Yps|Z{=}0, \Yps|Z{=}1) \le \alpha \cdot \dtv(\Yos|Z{=}0, \Yos|Z{=}1) = \dtv(\Ycs|Z{=}0, \Ycs|Z{=}1), \]
	which is what we want.
	
	For the second part of the theorem, an optimal model has $\Ycs = \Yps$, so we can substitute the $\Ycs$ in \eqref{eqn:hybrid} with $\Yps$ to get
	\[ \dtv(\Yps|Z{=}0, \Yps|Z{=}1) = \alpha \cdot \dtv(\Yos|Z{=}0, \Yos|Z{=}1). \]
	This is simply the equality in \eqref{eqn:disparity}, so we are done.
\end{proof}

The $\alpha$-disparity test is closely related to demographic parity and equalized odds.
0-disparity is satisfied if and only if the output disparity is zero, so it is equivalent to demographic parity.
In addition, we can easily adapt the proof of Theorem~\ref{thm:dispampwysiwyg} to show that equalized odds implies 1-disparity.
However, because equalized odds imposes a condition for each possible value of $\Yos$, 1-disparity does not imply equalized odds.
Although it may thus seem that equalized odds is stronger and better than 1-disparity, recent results by Corbett-Davies and Goel~\citeyearpar{corbettdavies2018measure} show that the threshold rule, which they argue is optimal, does not lead to equalized odds in general.
Therefore, there is a trade-off between the stronger fairness guarantee provided by equalized odds and the higher utility that is attainable under 1-disparity.
Of course, the 1-disparity test has the additional benefit that it can be generalized to other values of $\alpha$.

We end this section with theorems describing the consequences of enforcing the $\alpha$-disparity test with a wrong value of $\alpha$.
These theorems are close analogues of Theorems~\ref{thm:notwae} and \ref{thm:notwysiwyg}, respectively.
\begin{theorem} \label{thm:smallalphatest}
	If the $\alpha$-Hybrid worldview holds, a model that passes the $\alpha'$-disparity test, with $\alpha > \alpha'$, has a construct accuracy at most $1 - \frac{1}{2} (\alpha - \alpha') \cdot \dtv(\Yos|Z{=}0, \Yos|Z{=}1)$.
\end{theorem}
\begin{proof}
	By the reasoning in the proof of Theorem~\ref{thm:notwae}, we have for all $z \in \binset$,
	$ \Pr[\Ycs{=}\Yps \mid Z{=}z] \le 1 - \dtv(\Ycs|Z{=}z, \Yps|Z{=}z) $,
	which can be rewritten as
	$ \Pr[\Ycs{\neq}\Yps \mid Z{=}z] \ge \dtv(\Ycs|Z{=}z, \Yps|Z{=}z) $.
	
	Thus, the construct \emph{inaccuracy} of the model is
	\begin{align*}
	&\textstyle \frac{1}{2} \big(\Pr[\Ycs{\neq}\Yps \mid Z{=}0] + \Pr[\Ycs{\neq}\Yps \mid Z{=}1]\big) \\
	&\textstyle \ge \frac{1}{2} \big(\dtv(\Ycs|Z{=}0, \Yps|Z{=}0) + \dtv(\Ycs|Z{=}1, \Yps|Z{=}1)\big) \\
	&\textstyle \ge \frac{1}{2} \big(\dtv(\Ycs|Z{=}0, \Ycs|Z{=}1) - \dtv(\Yps|Z{=}0, \Yps|Z{=}1)\big) \\
	&\textstyle \ge \frac{1}{2} (\alpha - \alpha') \cdot \dtv(\Yos|Z{=}0, \Yos|Z{=}1),
	\end{align*}
	where the second inequality is an application of the triangle inequality and the third follows from the definitions of the $\alpha$-Hybrid worldview and the $\alpha'$-disparity test.
	
	Therefore, the construct accuracy, which is one minus the construct inaccuracy, is at most $1 - \frac{1}{2} (\alpha - \alpha') \cdot \dtv(\Yos|Z{=}0, \Yos|Z{=}1)$.
\end{proof}
\begin{theorem} \label{thm:bigalphatest}
	If the $\alpha$-Hybrid worldview holds, a model that passes the $\alpha'$-disparity test, with $\alpha < \alpha'$, can still have disparity amplification.
\end{theorem}
\begin{proof}
	The $\alpha'$-disparity test ensures that 
	\[ \dtv(\Yps|Z{=}0, \Yps|Z{=}1) \le \alpha' \cdot \dtv(\Yos|Z{=}0, \Yos|Z{=}1), \]
	and if equality holds here, we have
	\[ \dtv(\Yps|Z{=}0, \Yps|Z{=}1) = \alpha' \cdot \dtv(\Yos|Z{=}0, \Yos|Z{=}1) > \alpha \cdot \dtv(\Yos|Z{=}0, \Yos|Z{=}1) \]
	whenever $\dtv(\Yos|Z{=}0, \Yos|Z{=}1) \neq 0$.
	By the $\alpha$-Hybrid worldview, the rightmost quantity equals $\dtv(\Ycs|Z{=}0, \Ycs|Z{=}1)$, making the above inequality exactly that of disparity amplification (see~\eqref{eqn:dispamp}).
\end{proof}

\section{A More General Notion of Disparity Amplification} \label{sec:general}
In this section, we present a more general definition of disparity amplification that is a broader discrimination criterion and is applicable to numerical $\Ycs$.
Definition~\ref{def:dispamp} allows an output disparity if there \emph{exists} an equally large disparity in $\Ycs$, but it does not explicitly reflect the fact that we care about \emph{how} the model came to exhibit the disparity.
The only reason why we allow the disparity is that $\Ycs$ is the right attribute to use.
Thus, if the model does not use $\Ycs$ at all, then there should be no output disparity.
More formally, we want that if $\Ycs \perp \Yps$, then $\Yps \perp Z$.

Definition~\ref{def:dispampcont} generalizes this requirement and, unlike Definition~\ref{def:dispamp}, is applicable for both categorical and numerical $\Ycs$ at the expense of limiting $\Yps$ to be binary.
The generalization deals with cases where $\Yps$ is not independent of $\Ycs$ by measuring how much $\Yps$ depends upon $\Ycs$.
For binary $\Yps$, this dependence is captured by the likelihood function $\ell(\ycs) = \Pr[\Yps{=}1 \mid \Ycs{=}\ycs]$, and we use the Lipschitz continuity of this function to measure the dependence.
\begin{definition}[Disparity Amplification, Stronger] \label{def:dispampcont}
	For $\Yps \in \binset$ and $\ell(\ycs) = \Pr[\Yps{=}1 \mid \Ycs{=}\ycs]$, let $\rhostar$ be the smallest nonnegative $\rho$ such that $\ell$ is $\rho$-Lipschitz continuous.%
	\footnote{Technically, $\rhostar$ should be the \emph{infimum} of all $\rho$ such that $\ell$ is $\rho$-Lipschitz continuous, but it is not difficult to show then that $\ell$ is in fact $\rhostar$-Lipschitz continuous.}
	Then, a model exhibits \emph{disparity amplification} if
	\begin{equation} \label{eqn:dispampcont}
	\dtv(\Yps|Z{=}0, \Yps|Z{=}1) > \rhostar \cdot \dem(\Ycs|Z{=}0, \Ycs|Z{=}1).
	\end{equation}
\end{definition}
$\rhostar$ characterizes how much impact $\Ycs$ can have on the output of the model.
If the impact is small, we can conclude that the model is not using $\Ycs$ much, so not much output disparity can be explained by $\Ycs$.
On the other hand, if a small change in $\Ycs$ can cause a large change in the probability distribution of $\Yps$, then even a large output disparity can possibly be due to a small inter-group difference in $\Ycs$.
In fact, the use of $\rhostar$ makes Definition~\ref{def:dispampcont} invariant to scaling in $\Ycs$.
If a numerical $\Ycs$ is increased by some factor, $\rhostar$ will decrease by the same factor, so the quantity on the right-hand side of \eqref{eqn:dispampcont} will not change.

We now show relationships between the new Definition~\ref{def:dispampcont} and the previous definition (Definition~\ref{def:dispamp}).
First, we show that the old definition combined with a reasonable distance metric implies the new definition.
The previous definition assumes that $\Ycs$ is categorical, and in this case a natural distance metric for its support $\Ycalcs$ is the indicator $d(u, v) = \mathds{1}(u \neq v)$.
With this distance metric, we can relate the total variation distance used in the right-hand side of \eqref{eqn:dispamp} with the earthmover distance used in \eqref{eqn:dispampcont}.
\begin{theorem} \label{thm:dispamptocont}
	Let the construct $\Ycs$ be categorical with support $\Ycalcs$, which has distance metric $d(u, v) = \mathds{1}(u \neq v)$.
	If a model has disparity amplification under Definition~\ref{def:dispamp}, the model has disparity amplification under Definition~\ref{def:dispampcont} as well.
\end{theorem}
\ifunabridged
\begin{proof}
	We proceed by showing that $\rhostar \cdot \dem(\Ycs|Z{=}0, \Ycs|Z{=}1) \le \dtv(\Ycs|Z{=}0, \Ycs|Z{=}1)$.
	
	Since the likelihood function $\ell$ in Definition~\ref{def:dispampcont} is always between 0 and 1, we have $|\ell(u) - \ell(v)| \le 1 = d(u, v)$ when $u \neq v$, so $\ell$ is 1-Lipschitz continuous.
	Therefore $\rhostar \le 1$, and it suffices to show that $\dem(\Ycs|Z{=}0, \Ycs|Z{=}1) \le \dtv(\Ycs|Z{=}0, \Ycs|Z{=}1)$.
	
	By \citep[Theorem 4]{gibbs2002choosing}, we get
	\begin{align*}
	\dem(\Ycs|Z{=}0, \Ycs|Z{=}1)
	&\le \Big(\max_{u,v \in \Ycalcs} d(u, v)\Big) \cdot \dtv(\Ycs|Z{=}0, \Ycs|Z{=}1) \\
	&= \dtv(\Ycs|Z{=}0, \Ycs|Z{=}1),
	\end{align*}
	so we are done.
\end{proof}
\fi

Second, we show that Theorems~\ref{thm:dispampwae} and \ref{thm:dispampwysiwyg} still hold under the refined definition of disparity amplification.
Since the definitions of optimality and the empirical tests have not changed, we focus strictly on the nondiscrimination portions of the theorems.
\begin{theorem} \label{thm:dispampcontwae}
	A model that passes the demographic parity test does not have disparity amplification under Definition~\ref{def:dispampcont}.
\end{theorem}
\ifunabridged
The proof of Theorem~\ref{thm:dispampcontwae} is very similar to that of Theorem~\ref{thm:dispampwae} and will thus be omitted.
\fi
\begin{theorem} \label{thm:dispampcontwysiwyg}
	If the WYSIWYG worldview holds, then a model that passes the equalized odds test does not have disparity amplification under Definition~\ref{def:dispampcont}.
\end{theorem}
\ifunabridged
\begin{proof}
	We present the proof for the case where $\Ycs$ is continuous, but the proof for the discrete case is very similar.
	Let $p_0$ and $p_1$ be the probability density functions of $\Ycs|Z{=}0$ and $\Ycs|Z{=}1$, respectively.
	By Kantorovich duality~\citep[Equation 5.4]{villani2008optimal}, we have
	\begin{multline} \label{eqn:duality}
	\dem(\Ycs|Z{=}0, \Ycs|Z{=}1) \\
	\ge \int_{\Ycalcs} \phi(v) \, p_1(v) \, dv - \int_{\Ycalcs} \psi(u) \, p_0(u) \, du
	\end{multline}
	for all $\phi$ and $\psi$ such that $\phi(v) - \psi(u) \le d(u,v)$ for all $u, v \in \Ycalcs$.
	We set $\phi(v) = \psi(v) = \ell(v) / \rhostar$, where $\ell$ and $\rhostar$ are defined as in Definition~\ref{def:dispampcont}.
	Then, $\phi(v) - \psi(u) = (\ell(v) - \ell(u))/\rhostar \le d(u, v)$ by Lipschitz continuity.
	Thus, \eqref{eqn:duality} applies and implies that
	\begin{multline} \label{eqn:dualityappl}
	\rhostar \cdot \dem(\Ycs|Z{=}0, \Ycs|Z{=}1) \\
	\ge \int_{\Ycalcs} \ell(v) \, p_1(v) \, dv - \int_{\Ycalcs} \ell(u) \, p_0(u) \, du.
	\end{multline}
	By the WYSIWYG worldview and equalized odds, we have $\ell(y) = \Pr[\Yps{=}1 \mid \Ycs{=}y] = \Pr[\Yps{=}1 \mid \Ycs{=}y, Z{=}0] = \Pr[\Yps{=}1 \mid \Ycs{=}y, Z{=}1]$.
	Therefore, we can use the law of total probability to rewrite the first term on the right-hand side of \eqref{eqn:dualityappl} as $\Pr[\Yps{=}1 \mid Z{=}1]$, and similarly the second term becomes $\Pr[\Yps{=}1 \mid Z{=}0]$.
	
	If we let $\phi(v) = \psi(v) = -\ell(v) / \rhostar$ in \eqref{eqn:duality} instead, we get $\rhostar \cdot \dem(\Ycs|Z{=}0, \Ycs|Z{=}1) \ge \Pr[\Yps{=}1 \mid Z{=}0] - \Pr[\Yps{=}1 \mid Z{=}1]$.
	Finally, combining this inequality with the previous one gives us
	\begin{align*}
	\rhostar \cdot \dem(\Ycs|Z{=}0, \Ycs|Z{=}1)
	&\ge \Big|\Pr[\Yps{=}1 \mid Z{=}0] - \Pr[\Yps{=}1 \mid Z{=}1]\Big| \\
	&= \dtv(\Yps|Z{=}0, \Yps|Z{=}1),
	\end{align*}
	which is what we want.
\end{proof}
\fi

We now discuss the tightness of the above result.
In the extreme case where $\ell$ is a step function over real-valued $\ycs$, $\rhostar$ is infinite, so we trivially have a lack of disparity amplification under Definition~\ref{def:dispampcont}.
Thus, to receive meaningful fairness guarantees from Theorem~\ref{thm:dispampcontwysiwyg}, we must make sure that $\rhostar$ is not too large.
One way to achieve this is to apply the function $\ell$ to the construct space and reason about the transformed construct space.
If any transformation of the construct space results in a finding of disparity amplification under Definition~\ref{def:dispampcont}, then it is evidence that there could be a problem with the model with respect to discrimination.
Let $\ytildecs = \ell(\ycs)$ be a value in the transformed construct space, and $\tilde{\ell}$ denote the likelihood function on this space.
Then,
\[ \tilde{\ell}(\ytildecs) = \Pr[\Yps{=}1 \mid \Ytildecs{=}\ytildecs] = \Pr[\Yps{=}1 \mid \Ycs{=}\ycs] = \ell(\ycs) = \ytildecs, \]
so the transformation ensures that $\rho^*_{\!\tilde{\ell}} = 1$.

\paragraph{Connection to the $\alpha$-Disparity Test}
When $\Ycs$ and $\Yos$ are numerical, a natural extension of the $\alpha$-disparity test (Definition~\ref{def:disparity}) is
\begin{equation} \label{eqn:disparitycont}
\dtv(\Yps|Z{=}0, \Yps|Z{=}1) \le \rhostar \cdot \alpha \cdot \dem(\Yos|Z{=}0, \Yos|Z{=}1).
\end{equation}
For this to work, Worldview~\ref{def:hybrid} would have to change to use the earthmover distance rather than the total variation distance.
Since the earthmover distance is defined over a distance metric, the parameter $\alpha$ is not very meaningful unless $\Ycs$ and $\Yos$ have the same scale.
As a result, here we consider the case where $\Ycs$ and $\Yos$ are defined over the same metric space $(\Ycal, d)$.

Unfortunately, \eqref{eqn:disparitycont} is still not an empirical test because $\rhostar$ is defined in terms of $\Ycs$.
As tempting as redefining $\rhostar$ in terms of $\Yos$ is, $\Ycs$ and $\Yos$ can have vastly different likelihood functions despite having the same disparity, so this new empirical test will not guarantee the lack of disparity amplification under Definition~\ref{def:dispampcont}.
We leave as future work the discovery of an empirical test for numerical $\Ycs$ and $\Yos$ that corresponds to the $\alpha$-Hybrid worldview.

\section{Conclusion} \label{sec:conclusion}
We showed that demographic parity and equalized odds are related through our construct-based discrimination criterion of disparity amplification, arguing that the difference between the two empirical tests boils down to one's worldview.
In addition, we proved that calibration is not robust to post-processing and that predictive parity allows a model with an arbitrarily large output disparity regardless of the worldview and the observed base rates.

Our work differs from much of the prior work in that we consider the construct as separate from the observed data.
In particular, we interpreted the existing fairness definitions as acting on the observed data, whereas the discrimination criterion was viewed as a property of the construct.
This bifurcation allowed us to handle the following issues simultaneously:
(a) prohibitions against disparate impact have exceptions such as a business necessity, but
(b) due to past discrimination, the observed data can be biased in an unjustified way.
It is the second of these points that motivates our use of worldviews to characterize how biased the observed data is.

To illustrate how this might work in practice, let us revisit the examples in Section~\ref{sec:notation}.
In Example~1, there are reasons to believe that the observed recidivism rate is a racially biased measurement of the actual reoffense rate.
In Example~2, for various socioeconomic reasons, some protected groups may have disproportionately many people who take longer than six years to graduate but are eventually considered successful in the university.
The $\alpha$-Hybrid worldview can characterize these real-world scenarios, and the value of $\alpha$ reflects one's beliefs about how much more biased the observed data is than the construct.
Then, a practitioner can apply the $\alpha$-disparity test as a substitute for demographic parity or equalized odds, with the value of $\alpha$ determined through social research and public dialogue.

\ifunabridged
\section*{Acknowledgments}
This material is based upon work supported by the National Science Foundation under Grant No.~1704985.
Any opinions, findings, and conclusions or recommendations expressed in this material are those of the authors and do not necessarily reflect the views of the National Science Foundation.
\fi
	
	\bibliographystyle{ACM-Reference-Format}
	\bibliography{biblio}


\begin{thebibliography}{36}


\ifx \showCODEN    \undefined \def \showCODEN     #1{\unskip}     \fi
\ifx \showDOI      \undefined \def \showDOI       #1{#1}\fi
\ifx \showISBNx    \undefined \def \showISBNx     #1{\unskip}     \fi
\ifx \showISBNxiii \undefined \def \showISBNxiii  #1{\unskip}     \fi
\ifx \showISSN     \undefined \def \showISSN      #1{\unskip}     \fi
\ifx \showLCCN     \undefined \def \showLCCN      #1{\unskip}     \fi
\ifx \shownote     \undefined \def \shownote      #1{#1}          \fi
\ifx \showarticletitle \undefined \def \showarticletitle #1{#1}   \fi
\ifx \showURL      \undefined \def \showURL       {\relax}        \fi
\providecommand\bibfield[2]{#2}
\providecommand\bibinfo[2]{#2}
\providecommand\natexlab[1]{#1}
\providecommand\showeprint[2][]{arXiv:#2}

\bibitem[\protect\citeauthoryear{Angwin and Larson}{Angwin and Larson}{2016}]%
        {angwin2016propublica}
\bibfield{author}{\bibinfo{person}{Julia Angwin} {and} \bibinfo{person}{Jeff
  Larson}.} \bibinfo{year}{2016}\natexlab{}.
\newblock \showarticletitle{{ProPublica} responds to company's critique of
  machine bias story}.
\newblock \bibinfo{journal}{\emph{ProPublica}} (\bibinfo{year}{2016}).
\newblock


\bibitem[\protect\citeauthoryear{Barocas and Selbst}{Barocas and
  Selbst}{2016}]%
        {barocas2016big}
\bibfield{author}{\bibinfo{person}{Solon Barocas} {and}
  \bibinfo{person}{Andrew~D Selbst}.} \bibinfo{year}{2016}\natexlab{}.
\newblock \showarticletitle{Big data's disparate impact}.
\newblock \bibinfo{journal}{\emph{California Law Review}}
  \bibinfo{volume}{104} (\bibinfo{year}{2016}), \bibinfo{pages}{671--732}.
\newblock


\bibitem[\protect\citeauthoryear{Binns}{Binns}{2020}]%
        {binns2020apparent}
\bibfield{author}{\bibinfo{person}{Reuben Binns}.}
  \bibinfo{year}{2020}\natexlab{}.
\newblock \showarticletitle{On the apparent conflict between individual and
  group fairness}. In \bibinfo{booktitle}{\emph{ACM Conference on Fairness,
  Accountability, and Transparency}}. \bibinfo{pages}{514--524}.
\newblock


\bibitem[\protect\citeauthoryear{Calders, Kamiran, and Pechenizkiy}{Calders
  et~al\mbox{.}}{2009}]%
        {calders2009building}
\bibfield{author}{\bibinfo{person}{Toon Calders}, \bibinfo{person}{Faisal
  Kamiran}, {and} \bibinfo{person}{Mykola Pechenizkiy}.}
  \bibinfo{year}{2009}\natexlab{}.
\newblock \showarticletitle{Building classifiers with independency
  constraints}. In \bibinfo{booktitle}{\emph{IEEE International Conference on
  Data Mining Workshops}}. \bibinfo{pages}{13--18}.
\newblock


\bibitem[\protect\citeauthoryear{Calders and Verwer}{Calders and
  Verwer}{2010}]%
        {calders2010three}
\bibfield{author}{\bibinfo{person}{Toon Calders} {and} \bibinfo{person}{Sicco
  Verwer}.} \bibinfo{year}{2010}\natexlab{}.
\newblock \showarticletitle{Three naive Bayes approaches for
  discrimination-free classification}.
\newblock \bibinfo{journal}{\emph{Data Mining and Knowledge Discovery}}
  \bibinfo{volume}{21}, \bibinfo{number}{2} (\bibinfo{year}{2010}),
  \bibinfo{pages}{277--292}.
\newblock


\bibitem[\protect\citeauthoryear{Chouldechova}{Chouldechova}{2017}]%
        {chouldechova2017fair}
\bibfield{author}{\bibinfo{person}{Alexandra Chouldechova}.}
  \bibinfo{year}{2017}\natexlab{}.
\newblock \showarticletitle{Fair Prediction with Disparate Impact: A Study of
  Bias in Recidivism Prediction Instruments}.
\newblock \bibinfo{journal}{\emph{Big Data}} \bibinfo{volume}{5},
  \bibinfo{number}{2} (\bibinfo{year}{2017}), \bibinfo{pages}{153--163}.
\newblock


\bibitem[\protect\citeauthoryear{Corbett-Davies and Goel}{Corbett-Davies and
  Goel}{2018}]%
        {corbettdavies2018measure}
\bibfield{author}{\bibinfo{person}{Sam Corbett-Davies} {and}
  \bibinfo{person}{Sharad Goel}.} \bibinfo{year}{2018}\natexlab{}.
\newblock \showarticletitle{The Measure and Mismeasure of Fairness: A Critical
  Review of Fair Machine Learning}.
\newblock \bibinfo{journal}{\emph{arXiv}}  \bibinfo{volume}{1808.00023}
  (\bibinfo{year}{2018}).
\newblock


\bibitem[\protect\citeauthoryear{Darlington}{Darlington}{1971}]%
        {darlington1971another}
\bibfield{author}{\bibinfo{person}{Richard~B Darlington}.}
  \bibinfo{year}{1971}\natexlab{}.
\newblock \showarticletitle{Another Look at ``Cultural Fairness''}.
\newblock \bibinfo{journal}{\emph{Journal of Educational Measurement}}
  \bibinfo{volume}{8}, \bibinfo{number}{2} (\bibinfo{year}{1971}),
  \bibinfo{pages}{71--82}.
\newblock


\bibitem[\protect\citeauthoryear{Datta, Fredrikson, Ko, Mardziel, and
  Sen}{Datta et~al\mbox{.}}{2017}]%
        {datta2017proxy}
\bibfield{author}{\bibinfo{person}{Anupam Datta}, \bibinfo{person}{Matt
  Fredrikson}, \bibinfo{person}{Gihyuk Ko}, \bibinfo{person}{Piotr Mardziel},
  {and} \bibinfo{person}{Shayak Sen}.} \bibinfo{year}{2017}\natexlab{}.
\newblock \showarticletitle{Proxy Discrimination in Data-Driven Systems}.
\newblock \bibinfo{journal}{\emph{arXiv}}  \bibinfo{volume}{1707.08120}
  (\bibinfo{year}{2017}).
\newblock


\bibitem[\protect\citeauthoryear{Dieterich, Mendoza, and Brennan}{Dieterich
  et~al\mbox{.}}{2016}]%
        {compasrebuttal}
\bibfield{author}{\bibinfo{person}{William Dieterich},
  \bibinfo{person}{Christina Mendoza}, {and} \bibinfo{person}{Tim Brennan}.}
  \bibinfo{year}{2016}\natexlab{}.
\newblock \bibinfo{title}{{COMPAS} risk scales: Demonstrating accuracy equity
  and predictive parity}.
\newblock
  \bibinfo{howpublished}{\url{http://go.volarisgroup.com/rs/430-MBX-989/images/ProPublica_Commentary_Final_070616.pdf}}.
\newblock


\bibitem[\protect\citeauthoryear{Dwork, Hardt, Pitassi, Reingold, and
  Zemel}{Dwork et~al\mbox{.}}{2012}]%
        {dwork2012fairness}
\bibfield{author}{\bibinfo{person}{Cynthia Dwork}, \bibinfo{person}{Moritz
  Hardt}, \bibinfo{person}{Toniann Pitassi}, \bibinfo{person}{Omer Reingold},
  {and} \bibinfo{person}{Richard Zemel}.} \bibinfo{year}{2012}\natexlab{}.
\newblock \showarticletitle{Fairness through awareness}. In
  \bibinfo{booktitle}{\emph{Innovations in Theoretical Computer Science}}.
  \bibinfo{pages}{214--226}.
\newblock


\bibitem[\protect\citeauthoryear{{Equal Employment Opportunities
  Commission}}{{Equal Employment Opportunities Commission}}{1978}]%
        {eeoc1978}
\bibfield{author}{\bibinfo{person}{{Equal Employment Opportunities
  Commission}}.} \bibinfo{year}{1978}\natexlab{}.
\newblock \bibinfo{title}{Uniform Guidelines on Employee Selection Procedures}.
\newblock \bibinfo{howpublished}{29 CFR Part 1607}.
\newblock


\bibitem[\protect\citeauthoryear{Feldman, Friedler, Moeller, Scheidegger, and
  Venkatasubramanian}{Feldman et~al\mbox{.}}{2015}]%
        {feldman2015certifying}
\bibfield{author}{\bibinfo{person}{Michael Feldman}, \bibinfo{person}{Sorelle~A
  Friedler}, \bibinfo{person}{John Moeller}, \bibinfo{person}{Carlos
  Scheidegger}, {and} \bibinfo{person}{Suresh Venkatasubramanian}.}
  \bibinfo{year}{2015}\natexlab{}.
\newblock \showarticletitle{Certifying and removing disparate impact}. In
  \bibinfo{booktitle}{\emph{ACM SIGKDD International Conference on Knowledge
  Discovery and Data Mining}}. \bibinfo{pages}{259--268}.
\newblock


\bibitem[\protect\citeauthoryear{Friedler, Scheidegger, and
  Venkatasubramanian}{Friedler et~al\mbox{.}}{2016}]%
        {friedler2016impossibility}
\bibfield{author}{\bibinfo{person}{Sorelle~A Friedler}, \bibinfo{person}{Carlos
  Scheidegger}, {and} \bibinfo{person}{Suresh Venkatasubramanian}.}
  \bibinfo{year}{2016}\natexlab{}.
\newblock \showarticletitle{On the (im)possibility of fairness}.
\newblock \bibinfo{journal}{\emph{arXiv}}  \bibinfo{volume}{1609.07236}
  (\bibinfo{year}{2016}).
\newblock


\bibitem[\protect\citeauthoryear{Friedler, Scheidegger, Venkatasubramanian,
  Choudhary, Hamilton, and Roth}{Friedler et~al\mbox{.}}{2019}]%
        {friedler2019comparative}
\bibfield{author}{\bibinfo{person}{Sorelle~A Friedler}, \bibinfo{person}{Carlos
  Scheidegger}, \bibinfo{person}{Suresh Venkatasubramanian},
  \bibinfo{person}{Sonam Choudhary}, \bibinfo{person}{Evan~P Hamilton}, {and}
  \bibinfo{person}{Derek Roth}.} \bibinfo{year}{2019}\natexlab{}.
\newblock \showarticletitle{A comparative study of fairness-enhancing
  interventions in machine learning}. In \bibinfo{booktitle}{\emph{ACM
  Conference on Fairness, Accountability, and Transparency}}.
  \bibinfo{pages}{329--338}.
\newblock


\bibitem[\protect\citeauthoryear{Gibbs and Su}{Gibbs and Su}{2002}]%
        {gibbs2002choosing}
\bibfield{author}{\bibinfo{person}{Alison~L Gibbs} {and}
  \bibinfo{person}{Francis~Edward Su}.} \bibinfo{year}{2002}\natexlab{}.
\newblock \showarticletitle{On choosing and bounding probability metrics}.
\newblock \bibinfo{journal}{\emph{International Statistical Review}}
  \bibinfo{volume}{70}, \bibinfo{number}{3} (\bibinfo{year}{2002}),
  \bibinfo{pages}{419--435}.
\newblock


\bibitem[\protect\citeauthoryear{Green and Hu}{Green and Hu}{2018}]%
        {green2018myth}
\bibfield{author}{\bibinfo{person}{Ben Green} {and} \bibinfo{person}{Lily Hu}.}
  \bibinfo{year}{2018}\natexlab{}.
\newblock \bibinfo{title}{The Myth in the Methodology: Towards a
  Recontextualization of Fairness in Machine Learning}.
\newblock \bibinfo{howpublished}{Presented at the Machine Learning: The Debates
  workshop at the 35th International Conference on Machine Learning}.
\newblock


\bibitem[\protect\citeauthoryear{Grover}{Grover}{1995}]%
        {grover1995business}
\bibfield{author}{\bibinfo{person}{Susan~S Grover}.}
  \bibinfo{year}{1995}\natexlab{}.
\newblock \showarticletitle{The business necessity defense in disparate impact
  discrimination cases}.
\newblock \bibinfo{journal}{\emph{Georgia Law Review}}  \bibinfo{volume}{30}
  (\bibinfo{year}{1995}), \bibinfo{pages}{387--430}.
\newblock


\bibitem[\protect\citeauthoryear{Hardt, Price, and Srebro}{Hardt
  et~al\mbox{.}}{2016}]%
        {hardt2016equality}
\bibfield{author}{\bibinfo{person}{Moritz Hardt}, \bibinfo{person}{Eric Price},
  {and} \bibinfo{person}{Nati Srebro}.} \bibinfo{year}{2016}\natexlab{}.
\newblock \showarticletitle{Equality of opportunity in supervised learning}. In
  \bibinfo{booktitle}{\emph{Advances in Neural Information Processing
  Systems}}. \bibinfo{pages}{3315--3323}.
\newblock


\bibitem[\protect\citeauthoryear{Heidari, Loi, Gummadi, and Krause}{Heidari
  et~al\mbox{.}}{2019}]%
        {heidari2019moral}
\bibfield{author}{\bibinfo{person}{Hoda Heidari}, \bibinfo{person}{Michele
  Loi}, \bibinfo{person}{Krishna~P Gummadi}, {and} \bibinfo{person}{Andreas
  Krause}.} \bibinfo{year}{2019}\natexlab{}.
\newblock \showarticletitle{A Moral Framework for Understanding Fair {ML}
  through Economic Models of Equality of Opportunity}. In
  \bibinfo{booktitle}{\emph{ACM Conference on Fairness, Accountability, and
  Transparency}}. \bibinfo{pages}{181--190}.
\newblock


\bibitem[\protect\citeauthoryear{Kamishima, Akaho, Asoh, and Sakuma}{Kamishima
  et~al\mbox{.}}{2012}]%
        {kamishima2012fairness}
\bibfield{author}{\bibinfo{person}{Toshihiro Kamishima},
  \bibinfo{person}{Shotaro Akaho}, \bibinfo{person}{Hideki Asoh}, {and}
  \bibinfo{person}{Jun Sakuma}.} \bibinfo{year}{2012}\natexlab{}.
\newblock \showarticletitle{Fairness-aware classifier with prejudice remover
  regularizer}. In \bibinfo{booktitle}{\emph{Joint European Conference on
  Machine Learning and Knowledge Discovery in Databases}}.
  \bibinfo{pages}{35--50}.
\newblock


\bibitem[\protect\citeauthoryear{Kannan, Roth, and Ziani}{Kannan
  et~al\mbox{.}}{2019}]%
        {kannan2019downstream}
\bibfield{author}{\bibinfo{person}{Sampath Kannan}, \bibinfo{person}{Aaron
  Roth}, {and} \bibinfo{person}{Juba Ziani}.} \bibinfo{year}{2019}\natexlab{}.
\newblock \showarticletitle{Downstream effects of affirmative action}. In
  \bibinfo{booktitle}{\emph{ACM Conference on Fairness, Accountability, and
  Transparency}}. \bibinfo{pages}{240--248}.
\newblock


\bibitem[\protect\citeauthoryear{Kleinberg, Mullainathan, and
  Raghavan}{Kleinberg et~al\mbox{.}}{2017}]%
        {kleinberg2017inherent}
\bibfield{author}{\bibinfo{person}{Jon Kleinberg}, \bibinfo{person}{Sendhil
  Mullainathan}, {and} \bibinfo{person}{Manish Raghavan}.}
  \bibinfo{year}{2017}\natexlab{}.
\newblock \showarticletitle{Inherent Trade-Offs in the Fair Determination of
  Risk Scores}. In \bibinfo{booktitle}{\emph{Innovations in Theoretical
  Computer Science}}. \bibinfo{pages}{43:1--43:23}.
\newblock


\bibitem[\protect\citeauthoryear{Lindvall}{Lindvall}{2002}]%
        {lindvall2002lectures}
\bibfield{author}{\bibinfo{person}{Torgny Lindvall}.}
  \bibinfo{year}{2002}\natexlab{}.
\newblock \bibinfo{booktitle}{\emph{Lectures on the coupling method}}.
\newblock \bibinfo{publisher}{Dover Publications}.
\newblock


\bibitem[\protect\citeauthoryear{Liptak}{Liptak}{2017}]%
        {liptak2017sent}
\bibfield{author}{\bibinfo{person}{Adam Liptak}.}
  \bibinfo{year}{2017}\natexlab{}.
\newblock \showarticletitle{Sent to Prison by a Software Program's Secret
  Algorithms}.
\newblock \bibinfo{journal}{\emph{The New York Times}} (\bibinfo{year}{2017}).
\newblock


\bibitem[\protect\citeauthoryear{Liu, Dean, Rolf, Simchowitz, and Hardt}{Liu
  et~al\mbox{.}}{2018}]%
        {liu2018delayed}
\bibfield{author}{\bibinfo{person}{Lydia Liu}, \bibinfo{person}{Sarah Dean},
  \bibinfo{person}{Esther Rolf}, \bibinfo{person}{Max Simchowitz}, {and}
  \bibinfo{person}{Moritz Hardt}.} \bibinfo{year}{2018}\natexlab{}.
\newblock \showarticletitle{Delayed Impact of Fair Machine Learning}. In
  \bibinfo{booktitle}{\emph{International Conference on Machine Learning}}.
  \bibinfo{pages}{3156--3164}.
\newblock


\bibitem[\protect\citeauthoryear{Mueller}{Mueller}{2018}]%
        {mueller2018using}
\bibfield{author}{\bibinfo{person}{Benjamin Mueller}.}
  \bibinfo{year}{2018}\natexlab{}.
\newblock \showarticletitle{Using Data to Make Sense of a Racial Disparity in
  {NYC} Marijuana Arrests}.
\newblock \bibinfo{journal}{\emph{The New York Times}} (\bibinfo{year}{2018}).
\newblock


\bibitem[\protect\citeauthoryear{Narayanan}{Narayanan}{2018}]%
        {narayanan2018translation}
\bibfield{author}{\bibinfo{person}{Arvind Narayanan}.}
  \bibinfo{year}{2018}\natexlab{}.
\newblock \bibinfo{title}{Translation Tutorial: 21 Fairness Definitions and
  their Politics}.
\newblock \bibinfo{howpublished}{Tutorial at the first Conference on Fairness,
  Accountability, and Transparency. Abstract available at
  \url{https://facctconference.org/static/tutorials/narayanan-21defs18.pdf}.
  Recording available at \url{https://www.youtube.com/watch?v=jIXIuYdnyyk}}.
\newblock


\bibitem[\protect\citeauthoryear{Rawls}{Rawls}{1971}]%
        {rawls1971theory}
\bibfield{author}{\bibinfo{person}{John Rawls}.}
  \bibinfo{year}{1971}\natexlab{}.
\newblock \bibinfo{booktitle}{\emph{A theory of justice}}.
\newblock \bibinfo{publisher}{Harvard University Press}.
\newblock


\bibitem[\protect\citeauthoryear{Roemer}{Roemer}{2002}]%
        {roemer2002equality}
\bibfield{author}{\bibinfo{person}{John~E Roemer}.}
  \bibinfo{year}{2002}\natexlab{}.
\newblock \showarticletitle{Equality of opportunity: A progress report}.
\newblock \bibinfo{journal}{\emph{Social Choice and Welfare}}
  \bibinfo{volume}{19}, \bibinfo{number}{2} (\bibinfo{year}{2002}),
  \bibinfo{pages}{455--471}.
\newblock


\bibitem[\protect\citeauthoryear{{Supreme Court of the United States}}{{Supreme
  Court of the United States}}{2009}]%
        {ricci2009}
\bibfield{author}{\bibinfo{person}{{Supreme Court of the United States}}.}
  \bibinfo{year}{2009}\natexlab{}.
\newblock \bibinfo{title}{\textit{Ricci v.\ DeStefano}}.
\newblock \bibinfo{howpublished}{557 U.S. 557}.
\newblock


\bibitem[\protect\citeauthoryear{{Supreme Court of Wisconsin}}{{Supreme Court
  of Wisconsin}}{2016}]%
        {loomis2016}
\bibfield{author}{\bibinfo{person}{{Supreme Court of Wisconsin}}.}
  \bibinfo{year}{2016}\natexlab{}.
\newblock \bibinfo{title}{\textit{State v.\ Loomis}}.
\newblock \bibinfo{howpublished}{881 N.W.2d 749}.
\newblock


\bibitem[\protect\citeauthoryear{Villani}{Villani}{2008}]%
        {villani2008optimal}
\bibfield{author}{\bibinfo{person}{C{\'e}dric Villani}.}
  \bibinfo{year}{2008}\natexlab{}.
\newblock \bibinfo{booktitle}{\emph{Optimal transport: old and new}}.
  \bibinfo{series}{Grundlehren der mathematischen Wissenschaften: Comprehensive
  Studies in Mathematics}, Vol.~\bibinfo{volume}{338}.
\newblock \bibinfo{publisher}{Springer-Verlag Berlin Heidelberg}.
\newblock


\bibitem[\protect\citeauthoryear{Zafar, Valera, Gomez~Rodriguez, and
  Gummadi}{Zafar et~al\mbox{.}}{2017a}]%
        {zafar2017fairness-www}
\bibfield{author}{\bibinfo{person}{Muhammad~Bilal Zafar},
  \bibinfo{person}{Isabel Valera}, \bibinfo{person}{Manuel Gomez~Rodriguez},
  {and} \bibinfo{person}{Krishna~P Gummadi}.} \bibinfo{year}{2017}\natexlab{a}.
\newblock \showarticletitle{Fairness beyond disparate treatment \& disparate
  impact: Learning classification without disparate mistreatment}. In
  \bibinfo{booktitle}{\emph{International Conference on World Wide Web}}.
  \bibinfo{pages}{1171--1180}.
\newblock


\bibitem[\protect\citeauthoryear{Zafar, Valera, Rogriguez, and Gummadi}{Zafar
  et~al\mbox{.}}{2017b}]%
        {zafar2017fairness-aistats}
\bibfield{author}{\bibinfo{person}{Muhammad~Bilal Zafar},
  \bibinfo{person}{Isabel Valera}, \bibinfo{person}{Manuel~Gomez Rogriguez},
  {and} \bibinfo{person}{Krishna~P Gummadi}.} \bibinfo{year}{2017}\natexlab{b}.
\newblock \showarticletitle{Fairness Constraints: Mechanisms for Fair
  Classification}. In \bibinfo{booktitle}{\emph{Artificial Intelligence and
  Statistics}}. \bibinfo{pages}{962--970}.
\newblock


\bibitem[\protect\citeauthoryear{Zemel, Wu, Swersky, Pitassi, and Dwork}{Zemel
  et~al\mbox{.}}{2013}]%
        {zemel2013learning}
\bibfield{author}{\bibinfo{person}{Rich Zemel}, \bibinfo{person}{Yu Wu},
  \bibinfo{person}{Kevin Swersky}, \bibinfo{person}{Toni Pitassi}, {and}
  \bibinfo{person}{Cynthia Dwork}.} \bibinfo{year}{2013}\natexlab{}.
\newblock \showarticletitle{Learning fair representations}. In
  \bibinfo{booktitle}{\emph{International Conference on Machine Learning}}.
  \bibinfo{pages}{325--333}.
\newblock


\end{thebibliography}
\end{document}